\documentclass[manuscript,screen]{acmart}

\copyrightyear{2022}
\acmYear{2022}
\setcopyright{acmcopyright}\acmConference[-]{-}{-}{-}
\acmBooktitle{-}
\acmPrice{}
\acmDOI{}
\acmISBN{}

\usepackage{enumerate}
\usepackage[toc,page]{appendix}
\usepackage{color,tabularx,graphicx,amsmath,float,subfig}
\usepackage{multirow}
\usepackage{multicol}
\usepackage{cleveref}
\usepackage{enumitem}
\usepackage{lipsum}
\usepackage{multirow}
\usepackage{caption}
\usepackage{xcolor,colortbl}
\usepackage{array}
\usepackage{xspace}
\usepackage[justification=centering]{caption}
\newcolumntype{H}{>{\setbox0=\hbox\bgroup}c<{\egroup}@{}}

\crefformat{section}{\S#2#1#3} 
\crefformat{subsection}{\S#2#1#3}
\crefformat{subsubsection}{\S#2#1#3}

\definecolor{Gray}{gray}{0.85}

\AtBeginDocument{%
  \providecommand\BibTeX{{%
    \normalfont B\kern-0.5em{\scshape i\kern-0.25em b}\kern-0.8em\TeX}}}

\newcommand*{\eg}{e.g.,\xspace}
\newcommand*{\ie}{i.e.,\xspace}

\newcommand*{\fpr}{\textsf{FPR}\xspace}
\newcommand*{\fnr}{\textsf{FNR}\xspace}
\newcommand*{\ppv}{\textsf{PPV}\xspace}
\newcommand*{\acc}{\textsf{ACC}\xspace}
\newcommand*{\prev}{\textsf{p}\xspace}
\newcommand*{\grp}{\textsf{g}\xspace}
\newcommand*{\rc}{\textsf{k}\xspace}

\begin{document}


\title[The Possibility of Fairness]{The Possibility of Fairness: Revisiting the Impossibility Theorem in Practice}

\author{Andrew Bell}
\affiliation{%
  \institution{New York University}
  \streetaddress{50 West 4th St}
  \city{New York}
  \country{United States}}
\email{alb9742@nyu.edu}

\author{Lucius Bynum}
\affiliation{%
  \institution{New York University}
  \streetaddress{50 West 4th St}
  \city{New York}
  \country{United States}}
\email{lucius@nyu.edu}

\author{Nazarii Drushchak}
\affiliation{%
  \institution{Ukrainian Catholic University}
  \city{Lviv}
  \country{Ukraine}}
\email{naz2001r@gmail.com}

\author{Tetiana Herasymova}
\affiliation{%
  \institution{Ukrainian Catholic University}
  \city{Lviv}
  \country{Ukraine}}
\email{t.herasymova@ucu.edu.ua}

\author{Lucas Rosenblatt}
\affiliation{%
  \institution{New York University}
  \streetaddress{50 West 4th St}
  \city{New York}
  \country{United States}}
\email{lucas.rosenblatt@nyu.edu}

\author{Julia Stoyanovich}
\affiliation{%
  \institution{New York University}
  \streetaddress{50 West 4th St}
  \city{New York}
  \country{United States}}
\email{stoyanovich@nyu.edu}


\begin{abstract}

The ``impossibility theorem'' --- which is considered foundational in algorithmic fairness literature --- asserts that there must be trade-offs between common notions of fairness and performance when fitting statistical models, except in two special cases: when the prevalence of the outcome being predicted is equal across groups, or when a perfectly accurate predictor is used. However, theory does not always translate to practice. In this work, we challenge the implications of the impossibility theorem in practical settings. First, we show analytically that, by slightly relaxing the impossibility theorem (to accommodate a \textit{practitioner's} perspective of fairness), it becomes possible to identify a large set of models that satisfy seemingly incompatible fairness constraints. Second, we demonstrate the existence of these models through extensive experiments on five real-world datasets. We conclude by offering tools and guidance for practitioners to understand when --- and to what degree --- fairness along multiple criteria can be achieved. For example, if one allows only a small margin-of-error between metrics, there exists a large set of models simultaneously satisfying  \emph{False Negative Rate Parity}, \emph{False Positive Rate Parity}, and \emph{Positive Predictive Value Parity}, even when there is a moderate prevalence difference between groups. This work has an important implication for the community: achieving fairness along multiple metrics for multiple groups (and their intersections) is much more possible than was previously believed. 




\end{abstract}

\ccsdesc[500]{Computing methodologies~Machine learning}
\ccsdesc[500]{Social and professional topics~Socio-technical systems}
\ccsdesc[500]{Social and professional topics~Computing / technology policy}

\keywords{machine learning, fairness, public policy, responsible AI}

\maketitle

\section{Introduction}
\label{sec:introduction}

Increasingly, artificial intelligence (AI) and machine learning (ML) systems are being implemented in domains like employment, healthcare, and education to improve the efficiency of existing processes \cite{kuvcak2018machine, zejnilovic2021machine, shailaja2018machine}. In tandem with this uptick in adoption, there are growing concerns about the potential for ML systems to cause significant harm to members of already marginalized groups. For example, it has been found that some lending algorithms discriminate against Latinx and African-American borrowers \cite{BARTLETT2021, fuster2020predictably}, some prevalent medical algorithms discriminate against Black patients \cite{obermeyer2019dissecting}, and some educational risk-assessment algorithms perform worse for minority students \cite{hu_2020, Sapiezynski2017AcademicPP, obermeyer2019dissecting}.

The risk of discriminatory ML systems has led to significant interest in methods for measuring and ensuring ``algorithmic fairness.'' In the past decade, researchers have created robust processes and tools for auditing algorithmic systems for bias based on various definitions of fairness, such as \emph{Demographic Parity, Equalized Odds Ratios}, and \emph{Predictive Parity}~\cite{caton2020fairness,corbett2018measure, saleiro2018aequitas,lee2018detecting}. Choosing a context-specific fairness definition (also called a \emph{fairness metric}) depends on value judgments, and often several metrics may be situationally relevant. For instance, in contexts where the output of an algorithmic system is assistive, disparities in the \emph{False Negative Rate} between groups can be used as a measure of discrimination with respect to group need~\cite{saleiro2018aequitas}. 


In those contexts where more than one metric is applicable, practitioners, stakeholders, and the wider public may engage in a debate about which metric to choose~\cite{Washington2019HowTA}. Debates of this nature have yielded a number of notable results in the algorithmic fairness literature, including a fundamental result known colloquially as the ``impossibility theorem'' and simultaneously reported on by \citet{DBLP:journals/bigdata/Chouldechova17} and~\citet{DBLP:conf/innovations/KleinbergMR17}. The impossibility theorem asserts that, for binary classification, \emph{equalizing} some specific set of multiple common performance metrics between protected classes is impossible, except in two special cases. The first special case is when an algorithm is a \emph{perfect predictor}, and the second is when the \emph{prevalence} of the outcome being predicted (\emph{prevalence} is sometimes referred to as \emph{base rate}) is equal across groups. 
As a consequence of this theorem, researchers and practitioners have focused on understanding trade-offs between fairness and predictive accuracy in an algorithmic system, often designing bias audits and mitigation techniques that center on a single chosen fairness metric~\cite{corbett2018measure}.


Though important and strong, the implicit assumption of the impossibility result (namely, that a practitioner might think about fairness as \textit{exactly} equalizing metrics) may not actually apply to a wide array of real-world problems. In fact, a growing body of research suggests that the limitations to fairness derived by \citet{DBLP:conf/innovations/KleinbergMR17} and \citet{DBLP:journals/bigdata/Chouldechova17} may not be particularly relevant in many practical settings~\cite{rodolfa2021empirical, hsu2022pushing, wick2019unlocking, celis2019classification}.

\paragraph{Note:} 
Throughout this paper, we will often refer to metrics like \fpr (False Positive Rate), \fnr (False Negative Rate), \ppv (Precision) and \acc (Accuracy). Though these metrics are common, we seek to make our work  accessible across levels of technical expertise by providing equations and descriptions of these metrics in Appendix Section~\ref{appendix:metrics}. Relatedly, throughout this work we concern ourselves with \textit{binary classification}, a standard machine learning task where one attempts to assign the correct binary label (positive/negative) to each individual in a population. In fairness literature, it is common to consider at least two groups within that population, and then to compare the performance of a binary classifier on each sub-population. We also provide an in-depth definition of binary classification, and consideration of groups in the population, in Appendix Section~\ref{appendix:metrics}.

\paragraph*{Summary of contributions}
Variations on the ``impossibility theorem’’ specific to binary classification (with a protected class) state that, when equalizing certain metrics (like \fpr, \fnr, \ppv or \acc) between two groups, we should \text{hesitate} to consider multiple metrics at once. Why? The fairness constraints (equalizing three of these metrics between groups) will only be \textit{exactly} satisfiable if we have a perfect predictor or outcome prevalence parity \cite{DBLP:conf/innovations/KleinbergMR17,DBLP:journals/bigdata/Chouldechova17}. 

We suggest that this setting is unrealistic. Our paper's driving insight is that \textbf{practitioners are often more than comfortable with \textit{approximate} fairness guarantees, as opposed to enforcing \textit{exact} equality between metrics}. Therefore, we focus on a set of more realistic fairness constraints, where we are allowed to slightly relax between-group metric equalities for \fnr, \fpr, \ppv, and \acc. To our knowledge, we are the first to study at length this relaxed setting from a practitioner's point of view.\footnote{Previous work showed that a version of the impossibility result exists on the boundary of the relaxed setting \cite{DBLP:conf/innovations/KleinbergMR17}, but the authors did not fully explore the space of relaxed solutions, nor did they position it from the viewpoint of practitioners. This is further discussed in Section~\ref{sec:implications_of_impossibility_theorem}.} This framing yields a seemingly straightforward research question: under what type of setting and relaxation is it possible to find classifiers that are ``fair'' along seemingly incompatible fairness constraints? And how can practitioners determine if a ``fair'' classifier exists for \textit{their} predictive context?

For example, it turns out that if I, as a practitioner, say ``I have prevelances $p_1$ and $p_2$ for Groups 1 and 2 in dataset $X$, and I am willing to tolerate a difference of $Y$\% when equalizing metrics,'' 
then I have all of the information I need to determine if finding such a model is truly impossible (or not!) \emph{before even attempting the problem}. Encouragingly, and perhaps counter-intuitively, our analysis suggests that in practical settings the answer is often ``yes, it's possible to find that fair model.''
Our analysis suggests that if one allows only a small margin-of-error between metrics, there are large sets of models satisfying three fairness constraints simultaneously, even outside of perfect prediction and outcome prevalence parity. 
In our corresponding experiments on real datasets, we find empirically that the resource constraint \rc  (\ie having \rc loans to give out or \rc job interview slots to fill) also plays a significant role in feasibility, where a smaller \rc can result in more feasible models. 


\paragraph*{Paper roadmap.} We begin with background and related work in Section~\ref{sec:related}.  After, we approach the problem analytically in Section~\ref{sec:theory}. We state a formula balancing \fpr, \fnr and \acc between groups with fairness relaxations for each metric. In this setting, we are able to derive a powerful tool in the form of a simple formula relating the feasibility of fairness to a specific relaxation strength, given a classification scenario. However, in many resource-constrained settings, practitioners care more about \ppv than \acc. So, we next turn our attention to the problem in terms of \fpr, \fnr and \ppv, but find that it is difficult to analyze in closed-form. Instead, through principled approximations, we are able to provide much the same guidance to practitioners as a direct analytical solution would, and leave deriving a closed-form result to future work. {\bf For practitioners who wish to go directly to our margin-of-error based fairness feasibility recommendations, 
skip to Section~\ref{sec:discussion}}. Additionally, to demonstrate the utility of our fairness relaxation insights, we conduct extensive experimental evaluations on five real-world datasets. The results of these experiments, discussed in Section~\ref{sec:experiments}, corroborate our insights and compellingly demonstrate the \textit{possibility} of fairness. We discuss our insights and offer guidance to practitioners in Section~\ref{sec:discussion} and conclude in Section~\ref{sec:conclusion}. Our main take-away is that the rigidity of the impossibility theorem is based on theoretical assumptions about what constitutes classifier fairness. Through this work, we hope to contrast those results by exploring \textit{practitioner-focused} fairness assumptions.

\section{Background and Related Work}
\label{sec:related}

\paragraph*{Algorithmic fairness.} In the past decade, significant progress has been made in understanding algorithmic fairness \cite{mitchell2021algorithmic}. Broadly, this literature concludes that fairness is not a monolith: there are \textit{many} different ways to think about algorithmic fairness, and defining what is ``fair'' is a matter of philosophy, incorporating one's worldview, mitigation objectives, and an algorithm's context-of-use~\cite{khan2021fairness, friedler2016possibility}. In response to the complex and nuanced nature of fairness, researchers have defined dozens of \emph{fairness metrics}, or mathematical assessments of an algorithm's prejudice, that address different aspects of fairness~\cite{bird2020fairlearn, aif360-oct-2018, saleiro2018aequitas, fairnessRpackage, DBLP:journals/bigdata/Chouldechova17, calders2010three, friedler2019comparative, zafar2017fairness, mehrabi2021survey, verma2018fairness}. 
Broadly, these metrics can be divided into two categories: those that consider the output of an algorithm, and those that consider errors made by the algorithm. As an example of the former, \emph{Disparate Impact} (or \emph{Proportional Parity}) measures the proportion of a group receiving the positive classification outcome relative to the proportion of the group in the input. As an example of the latter, the difference in \emph{False Negative Rates} between groups can be used to assess whether one group is erroneously ``passed over'' for a positive outcome relative to another. Importantly, there is no one-size-fits-all metric for evaluating the fairness of algorithms. Some tools (like the \emph{Fairness Tree} \cite{saleiro2020dealing}) have been developed to help navigate the challenge of selecting an appropriate fairness metric, but ultimately, it is necessary for researchers and practitioners to have meaningful conversations with those impacted by algorithms to select fairness metric(s) specific to the context-of-use~\cite{saleiro2018aequitas, ruf2021towards}.


Typically, error-based metrics judge the fairness of a predictor by considering the \textit{imbalance} between group-specific metrics. We can calculate imbalance as a difference --- mean, squared, absolute, etc. --- or as a \emph{disparity} --- the ratio of a metric of one group, $g_j$, to that of a \emph{reference group}, $g_{ref}$, usually chosen as the majority group: 
    $\textit{disparity}_{g_j} = \frac{\textsf{metric}_{g_j}}{\textsf{metric}_{g_{ref}}}$.
Often, the goal of algorithmic fairness is to achieve \emph{parity}, that is, to \textit{eliminate} the imbalance between fairness metrics entirely. Importantly, the tolerable level of difference/disparity for a given fairness metric is highly dependent on the algorithm's context-of-use. Perhaps counter-intuitively, there are cases where we want to enforce a large disparity (see \citet{rodolfa2020case}, who discuss intentionally over-representing a marginalized group for an assistive intervention). 

While some fairness metrics are incompatible with one another\footnote{For example, one cannot simultaneously satisfy equal selection and proportional parity unless the prevalence of outcomes is the same in both groups}, and others are approximately mathematically equivalent~\cite{rosenblatt2022counterfactual}, many are compatible and distinct. Consider the following scenario: a high school is using an algorithm to predict which students are at risk of failing ninth grade, so that high-risk students can be offered a special tutoring intervention. School administrators may want an algorithm that selects an equal number of privileged and underprivileged students, and also does not unfairly pass over  students who are badly in need of tutoring. This would imply a need for both \emph{Demographic Parity} and \emph{False Negative Rate Parity} between groups.
Yet, as we argued in the introduction, multiple fairness metrics are rarely considered in practice, and most existing bias mitigation methods enforce a single metric (or at most two metrics) at a time~\cite{rodolfa2020case, feldman2015certifying, kamiran2009classifying, kamiran2010classification, kamiran2012data, zafar2017fairness, berk2017convex, hardt2016equality, pleiss2017fairness, nandy2022achieving}. In part, this is due to the \emph{impossibility theorem}, a foundational result, presented simultaneously by \citeauthor{DBLP:journals/bigdata/Chouldechova17} and by \citeauthor{DBLP:conf/innovations/KleinbergMR17}.

\paragraph{The impossibility theorem} As stated by \citet{DBLP:conf/innovations/KleinbergMR17}, this theorem shows that three common metrics --- equalizing calibration within groups, and enforcing balance for the negative class and for the positive class --- cannot be simultaneously satisfied for multiple groups, outside of two special cases~\cite{DBLP:conf/innovations/KleinbergMR17}. These cases are (1) when the algorithm is a perfect predictor and (2) when there is no prevalence difference between groups. 
\citet{DBLP:journals/bigdata/Chouldechova17} states an equivalent impossibility, 
presented as the relationship between the \emph{Predictive Positive Value} (\ppv), \emph{False Positive Rate} (\fpr), \emph{False Negative Rate} (\fnr), and \emph{prevalence} (\prev) (Equation~\ref{impossibility_equation}). 
\begin{align}
\label{impossibility_equation}
\fpr = \frac{\prev}{1-\prev} \frac{1-\ppv}{\ppv} (1 - \fnr)
\end{align}



%


\paragraph{Exploring implications of the impossibility theorem}
\label{sec:implications_of_impossibility_theorem}
Importantly, the impossibility results imply an upper bound on how many fairness metrics can be satisfied simultaneously without a perfect predictor. \citeauthor{DBLP:conf/innovations/KleinbergMR17} addressed a key question surrounding approximate conditions of the impossibility result, showing that approximate fairness definitions can simultaneously hold, but only under $\epsilon$-approximate prevelances or $\epsilon$-approximate perfect prediction~\cite{DBLP:conf/innovations/KleinbergMR17}. \emph{Significantly, \citeauthor{DBLP:conf/innovations/KleinbergMR17} did not explore the space of solutions under $\epsilon$-approximate relaxations of fairness constraints, nor did they detail the implications that these relaxations might have for practitioners}.

To motivate our exploration of this space, consider the following thought experiment. The \emph{achievement gap} is one of the most pervasive examples of racial disparities in education, in which Black and Brown students graduate from high school at a rate roughly 10\% lower than that of White students~\cite{langham2009achievement}. How should practitioners think about a prevalence difference of 10\% when designing algorithms that predict student performance?

Understanding the space of feasible models under a relaxation of the impossibility theorem is particularly salient in light of recent work showing that theoretical trade-offs do not always apply to real-world settings~\cite{rodolfa2021empirical, celis2019classification, wick2019unlocking}. For example, \citeauthor{rodolfa2021empirical} introduced a method for finding models that were fair with respect to \fnr without sacrificing a model's \ppv, and demonstrated the effectiveness of their approach in four separate ML-for-public-policy problems~\cite{rodolfa2021empirical}.\footnote{\citeauthor{rodolfa2021empirical} refer to \emph{Recall Parity} in their work, but state that it is mathematically equivalent to \fnr Parity for small population sizes.} It was hypothesized by the authors that the negligible trade-off is the result of the resource-constrained nature of applied ML problems, where fairness and model performance are measured with respect to the top-\emph{k}, rather than at an arbitrary threshold. In our work, we begin to formalize this intuition in Theorem~\ref{theorem:role_of_k}.

Other works also challenge the idea that accuracy and fairness are in tension. \citet{celis2019classification} developed a meta-algorithm for a large family of classification problems with convex constraints, and demonstrated that one can achieve near-perfect fairness while sacrificing only a small amount of accuracy. Similarly, \citet{wick2019unlocking} propose a semi-supervised learning approach that improves both fairness and accuracy. A third recent example is the \emph{MFOpt} framework proposed by \citet{hsu2022pushing} that simultaneously optimizes \emph{Demographic Parity}, \emph{Equalized Odds}, and \emph{Predictive Rate Parity}---those fairness notions that are mathematically incompatible according to the impossibility theorem. Similar to our work, \citeauthor{hsu2022pushing} were motivated by doubts about the strength of the impossibility theorem in practical settings.
Notably, these works have focused on methods for mitigating disparity for multiple fairness metrics while maintaining high model accuracy, but have \textit{not} provided much analysis of their implicit relaxing of fairness metric parity.

\paragraph{Practice versus theory.} 
In our analysis, we center two considerations common to practical settings. First, in practice, one generally does not require fairness metrics to be \emph{exactly equal} across groups to achieve fairness. For example, depending on the context of use, a classifier that has an \fpr difference between groups of $2\%, 5\%$ or even $10\%$ may be satisfactory. 
Second, we consider the presence of a resource constraint. For example, a commonly used performance metric in applied ML problems is \ppv-at-\rc, where \rc represents a real-world resource constraint~\cite{wilde_2021, carton2016identifying, aguiar2015and}. 

\section{Finding feasible models}
\label{sec:theory}

To encode that practitioners are generally okay with approximate fairness constraints as opposed to strict constraints, we begin by directly re-parameterizing the impossibility theorem with relaxations for each parameter. The relaxed constraints afford us a \emph{space of solutions}, where each solution represents 
a potential classifier that balances all three metrics within our desired tolerance. We call this space of solutions the {\bf fairness region}. Exploring how this region changes across different contexts/relaxations/metric settings can tell us when we  satisfying all constraints is feasible and, further, when we can expect greater flexibility when searching for a model across multiple metrics.

The impossibility theorem can be stated in terms of different metrics, and our choice impacts the ease or difficulty of characterizing the fairness region in closed-form. We start with a choice of metrics for which we have a closed-form characterization of the fairness region: \fnr, \fpr, and \acc (Section~\ref{sec:attempt_1}). Motivated by the fact that practitioners in resource constrained settings often consider \ppv instead of \acc, we then consider a fairness region for \fnr, \fpr, and \ppv (Section~\ref{sec:attempt_2}). Deriving a closed-form solution for the fairness region in the second case is much more difficult, requiring us to approach our analysis computationally.

\subsection{Characterizing the fairness region using \fpr, \fnr, and \acc}
\label{sec:attempt_1}



We begin by defining an alternative expression for the impossibility result, this time in terms of \fpr, \fnr, and \acc. Proofs of all results in this section (Corollary~\ref{corollary:acc_impossiblity}, Proposition~\ref{proposition:exp_fairness_area_acc}, and Theorem~\ref{lemma:area_closed_form_acc}) can be found in Appendix~\ref{appendix:imp_acc}.

\begin{corollary}[Impossibility Result Variation \cite{DBLP:journals/bigdata/Chouldechova17} \cite{DBLP:conf/innovations/KleinbergMR17}]
In a binary classification setting (see Appendix Section~\ref{appendix:metrics}), the relationship between \acc, \fnr, \fpr and \prev can be characterized by: \quad $\label{eq:acc_impossibility_result}
    \acc = (1 - \fnr)\prev + (1 - \fpr)(1 - \prev).$

\label{corollary:acc_impossiblity}
\end{corollary}

Next, we add a relaxation term for each parameter in Corollary~\ref{eq:acc_impossibility_result}.
In the case of two groups, we let $\fpr_2 = \fpr_1 + \epsilon_\fpr$, where $\epsilon_\fpr$ is a tolerable difference between the metric for the two groups. Similarly, 
let $\fnr_2 = \fnr_1 + \epsilon_\fnr$ and $\acc_2 = \acc_1 + \epsilon_\acc$. Using these relaxations, we can express a ``governing equation'' for the fairness region as follows.


\begin{proposition}[Describing The Fairness Region]
\label{proposition:exp_fairness_area_acc}
Consider Corollary~\ref{corollary:acc_impossiblity}. Assume that $\prev_2 = \prev_1 + \epsilon_{\prev}$, $\acc_2 = \acc_1 + \epsilon_{\acc}$, $\fpr_2 = \fpr_1 + \epsilon_{\fpr}$, and $\fnr_2 = \fnr_1 + \epsilon_{\fnr}$, where each $\epsilon_\fpr, \epsilon_\fnr, \epsilon_\acc, \epsilon_\prev \in (-1,1)$ term captures the difference between two groups for \prev, \acc, \fpr, and \fnr, respectively. Then, the following equality holds:
\begin{align}
\label{eq:variation_area_exp_body}
\fnr_1 = \frac{-\epsilon_\fpr + \epsilon_\acc + \epsilon_\fpr \cdot \prev_1 - \epsilon_\fnr \cdot \prev_1 + \fpr_1 \cdot \epsilon_\prev + \epsilon_\fpr \cdot \epsilon_\prev - \epsilon_\fnr \cdot \epsilon_\prev}{\epsilon_\prev}
\end{align}
\label{lemma:constraints}
\end{proposition}

While Equation~\ref{eq:variation_area_exp_body} may look complex, an important insight is that it shows $\fnr$ can be expressed as a function of mostly fixed and known terms. \textbf{Observe that \prev and $\epsilon_\prev$ are known \textit{a priori}, as they can be calculated directly from the dataset.} By deciding on bounds for the acceptable tolerance between fairness metrics (\ie maximum allowable values for $\epsilon_\fpr, \epsilon_\fnr, \epsilon_\acc$), we can then create plots of $\fnr$ vs. $\fpr$, as seen in Figure~\ref{fig:area_estimate} (a). Each point in these plots represents an \fpr, \fnr (and, implicitly, an \acc value) of a feasible model. In other words, these points correspond to the existence of feasible models satisfying fairness constraints for \fpr, \fnr, and \acc within an $\epsilon$-margin-of-error. Similarly, the absence of a point corresponds to the emptiness of a set of models (\ie the infeasibility of \emph{finding} a model). In general, we can say that plotting \fnr vs. \fpr according to Proposition~\ref{proposition:exp_fairness_area_acc} gives us a projection of the fairness region, where the size of the area provides a measure (out of the entire \fpr,\fnr$\in \lbrack 0,1 \rbrack$ region) of the proportion of feasible models that are fair across all three metrics of interest (out of all possible metric values). Significantly, we can use Equation~\ref{eq:variation_area_exp_body} to find a closed-form expression for the size of the fairness region over the unit square $\fnr, \fpr \in \lbrack 0,1 \rbrack $:

\begin{theorem}[Size of the Fairness Region] \label{lemma:area_closed_form_acc} 
Assume $\epsilon_\prev < 1 - \prev$ (a mild assumption). Allow $\pm \gamma$ to be the symmetric acceptable error (our ``fair'' relaxation) between groups for metrics \fpr, \fnr, and \acc. Consider the size of the space of possible $\epsilon_{\fpr}, \epsilon_{\fnr}, \epsilon_{\acc}$ assignments, given $\epsilon_\prev$ and \prev that satisfy the constraints from Proposition~\ref{proposition:exp_fairness_area_acc}. We will denote the size of that space as $|A_f|$ (as shorthand, we will call this the ``fairness region''). For a set of fairness constraints  $-\gamma \leq \epsilon_{\fpr}, \epsilon_{\fnr}, \epsilon_{\acc} \leq \gamma$, where $|\gamma| \leq 1$ and $\gamma \neq 0$, we have that $|A_f|$ is simply:
\begin{align}
    |A_f| = \frac{4 \gamma}{\epsilon_\prev} - \frac{4 \gamma^2}{{\epsilon_\prev}^2}
\end{align} 


\end{theorem}
The practical implication of Theorem~\ref{lemma:area_closed_form_acc} is simple: for a practitioner with target group-wise \fpr, \fnr, and \acc, we can show them whether their fairness relaxation values will work or not given their context (\ie their \prev and $\epsilon_\prev$), and, furthermore, how relaxing (or tightening) their $\epsilon$-margin-of-error affects the overall fairness region.




\subsection{Characterizing the fairness region using \fpr, \fnr, \ppv}\label{sec:attempt_2}

Theorem~\ref{lemma:area_closed_form_acc} provides a clean and convenient result for \fpr, \fnr and \acc, but it does not allow us to meaningfully analyze resource-constrained settings.
Generally, practitioners face resource-constrained scenarios where a classifier's \acc has less meaning than its \ppv \cite{bell2019proactive, rodolfa2021empirical, aguiar2015and, carton2016identifying, wilde_2021}. To this end, we attempted to recreate the analysis in Section~\ref{sec:attempt_1} instead using \fpr, \fnr, and \ppv, which is found in Appendix~\ref{appendix:region_ppv}. This analysis with \ppv instead of \acc leads us to an analogous expression for \fnr as a function of the other parameters (see \ref{lemma:relaxed_impossibility_result}). However, the expression for the \ppv case is ripe with non-linearities and possible discontinuities, making it more difficult to find a closed-form expression for the size of the fairness region (in the same way we did for \acc in Theorem~\ref{lemma:area_closed_form_acc}). A corresponding plot of the fairness region projected onto two dimensions (\fnr and \ppv) is shown in Figure~\ref{fig:area_estimate}~(b).


With no closed-form expression, we take two computational approaches to understanding the size of this fairness region. The first approach is to directly estimate the fraction of the unit square ($\fnr, \ppv \in [0, 1]$) taken up by a discretized feasible region by using a \textit{dot planimeter}, which is a well-studied method for estimating complex two-dimensional areas \cite{frolov1969accuracy,bocarov1957matematiko}. Figure~\ref{fig:area_estimate}~(b) is a discretized set of solutions created by sweeping out a range of parameter values and plotting feasible lines following the equation for \fnr. Intuitively, dot planimetry estimates the fraction of the unit square taken up by the set of solutions by overlaying a regular grid of points. For each point (also known as a \textit{detector}), we check whether or not any feasible lines pass within a specific distance tolerance, which is a function of the the grid's granularity. An example of this procedure is shown in the corresponding Figure~\ref{fig:area_estimate}~(c). Unfortunately, the process of dot-planimeter-style estimation introduces additional approximation error on top of discretizing the fairness region. Our analysis of upper-bounding this error (under some assumptions) can be found in Section~\ref{appendix:dot} of the Appendix. 

To avoid this additional approximation error, for our second approach we re-frame our description of the fairness region using a Constraint Program (CP). A constraint program provides an alternative means of measuring how large the space of feasible solutions is for a given setting of tolerances. Rather than measuring the area taken up by a projection on two dimensions (\ppv and \fnr), we can describe the fairness region directly as the set of \emph{feasible solutions} to a constraint program. Using the CP-SAT solver in Google ORTools, we express our problem's governing equations as a set of integer variables and constraints. Our quantities of interest ($\fpr, \fnr, \ppv, \prev, \epsilon_\fpr, \epsilon_\fnr, \epsilon_\ppv, \epsilon_\prev$) are all real numbers rather than integers, with infinitely many possible values in their respective ranges. To characterize the size of the solution space for different tolerances, we discretize the interval $[0, 1]$ into $N + 1$ bins (and, correspondingly, the interval $[-1, 1]$ into $2N + 1$ bins). For example, an $\fpr=0.91$  corresponds to an integer value of 91 when $N=100$. With this discretization, we represent the fairness region using the following constraint program:

\begin{figure}
\centering
\caption{$p_1 = 0.3, p_2 = 0.5$; $\epsilon_\fpr, \epsilon_\fnr, \epsilon_\acc \in \lbrack -0.05,0.05 \rbrack $}
\vspace{-0.1cm}
\subfloat[Fairness region when relating \\ \fpr, \fnr, \acc]
{\includegraphics[width=.33\textwidth]{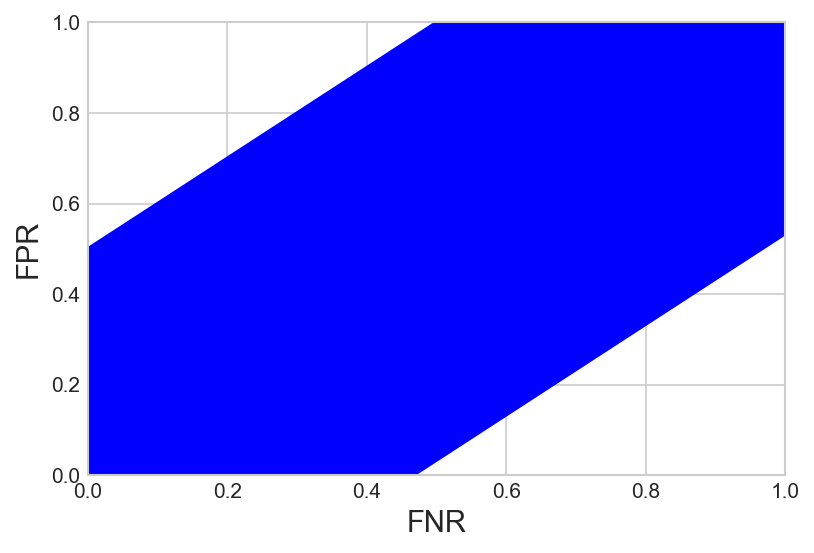}}
\subfloat[Fairness region when relating \\ \fpr, \fnr, \ppv]
{\includegraphics[width=.33\textwidth]{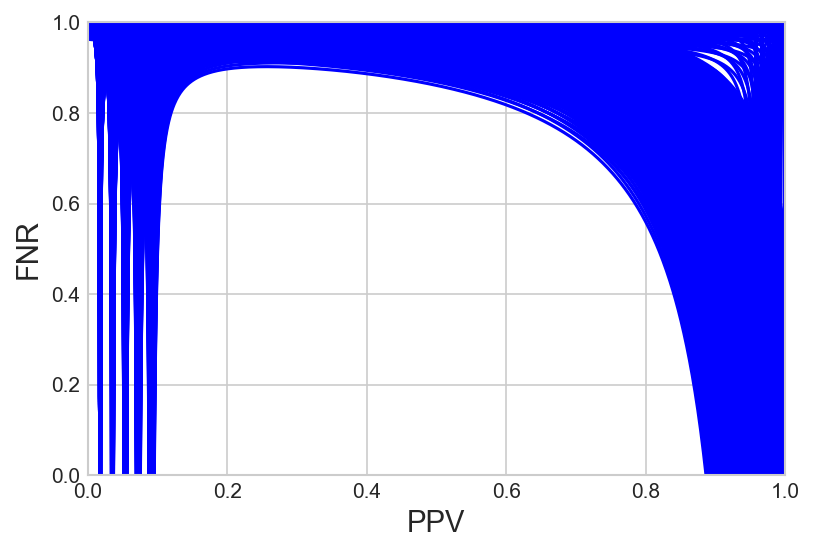}}
\subfloat[Estimating the size of (b) \\ with a dot planimeter]
{\includegraphics[width=.33\textwidth]{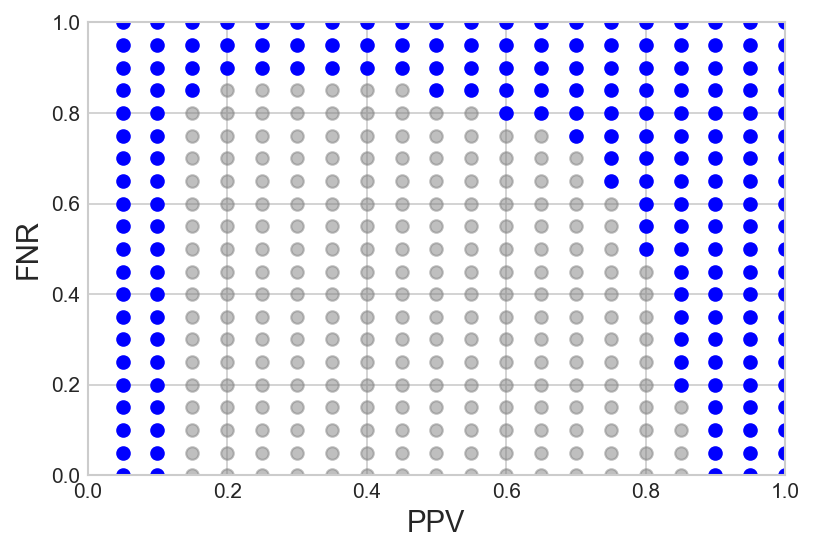}
}
\label{fig:area_estimate}
\end{figure}

\begin{equation*}
\begin{array}{rll}
    \alpha_i, \beta_i, p_i, v_i &\in [0, N] \ \text{integer} & \forall i \in \{1, 2\}\\
    \epsilon_j &\in [-N, N] \ \text{integer} & \forall j \in \{\alpha, \beta, p, v\}\\
    m_i, d_i &\in [0, N^2] \ \text{integer} & \forall i \in \{1, 2\}\\
    n_i &\in [0, N^3] \ \text{integer} & \forall i \in \{1, 2\}\\
    
    p_i &= b_i \cdot N & \forall i \in \{1, 2\}\\

    j_1 &= j_2 + \epsilon_{j} & \forall j \in \{\alpha, \beta, p, v\}\\
\end{array}
\hspace{1cm}\rule[-1.2cm]{0.02cm}{2.7cm}\hspace{1cm}
\begin{array}{rll}
    \epsilon_{j} &\geq -\epsilon_{\text{max}} \cdot N & \forall j \in \{\alpha, \beta, v\}\\
    \epsilon_{j} &\leq \epsilon_{\text{max}} \cdot N & \forall j \in \{\alpha, \beta, v\}\\

    m_i &= p_i \cdot (N - v_i) & \forall i \in \{1, 2\}\\
    n_i &= m_i \cdot (N - \beta_i) & \forall i \in \{1, 2\}\\
    d_i &= v_i \cdot (N - p_i) & \forall i \in \{1, 2\}\\
    n_i &= \alpha_i \cdot d_i & \forall i \in \{1, 2\}\\
\end{array}
\end{equation*}

Here, $N$ is the number of integers; $n, m, d$ are intermediate variables used to represent multiplicative constraints for the CP-SAT solver; $\epsilon_{\text{max}}$ represents the maximum allowable value of $|\epsilon_\alpha|, |\epsilon_\beta|, |\epsilon_v|$; $b_i$ represent the observed prevalences in the real-valued range $[0, 1]$; and , $\alpha, \beta, v$ represent \fpr, \fnr, \ppv, respectively. The CP-SAT solver allows us to enumerate all possible solutions to a constraint program. With $N^6$ possible values for the set $\{\fpr_1, \fpr_2, \fnr_1, \fnr_2, \ppv_1, \ppv_2\}$, for any fixed $N$, we can characterize the size of the discretized solution space as a function of changes to the other inputs simply as the number of feasible solutions. 

\subsection{Revisiting the impossibility theorem}

Recall that there are two known exceptions to the impossibility theorem: when the two groups' prevalence values are the same, and under perfect prediction~\cite{DBLP:conf/innovations/KleinbergMR17, DBLP:journals/bigdata/Chouldechova17}. However, given our formalization for the \textit{relaxed} case of fairness constraints, perhaps we should ask: to what \textit{degree} do the exceptions from the impossibility result apply? Specifically:



\begin{enumerate}
    \item How large can prevalence differences be (\eg $\epsilon_\prev \in \{ 1\%, 10\%, 50\% \}$) and still imply a large fairness region?
    \item How far can a model depart from perfect prediction (\eg $\ppv \in \{99\%, 75\%\}$) to have a large fairness region?
\end{enumerate}

\subsubsection{Varying prevalence difference}

First we explore the impact of varying the prevalence difference between two groups on the size of the fairness region, using the CP described in Section~\ref{sec:attempt_2}. The results of these experiments can be seen in Figure~\ref{fig:varying_e_p}, which are heatmaps plotting the number of feasible models for any pair of prevalence values $\prev_1, \prev_2$ over a range of values from 0.01 to 0.99. Figures~\ref{fig:varying_e_p} (a), (b), (c), and (d) correspond to settings where the allowable difference between metrics is $\epsilon \leq 0.0, 0.02, 0.05$, and $0.1$, respectively. Note that in each setting we fix performance such that $\fnr, \ppv \in \lbrack 0,0.99 \rbrack$ to avoid the pathological cases covered by Equation~\ref{impossibility_equation}.

Several important insights can be gleaned from Figure~\ref{fig:varying_e_p}. As expected, in the case where the $\epsilon$-margin-of-error is $0$, feasible models are only found on the diagonal, when prevalences are equal (implied by \cite{DBLP:conf/innovations/KleinbergMR17, DBLP:journals/bigdata/Chouldechova17}). Interestingly, we observe for all settings of $\epsilon$ that the number of feasible models is densest around $\prev_1 = \prev_2 = 0.5$. For example, the fairness region is larger when $\prev_1 = 0.4, \prev_2 = 0.5$ than when $\prev_1 = 0.1, \prev_2 = 0.2$, even though $\epsilon_\prev = 0.1$ in both cases.

As the $\epsilon$-margin-of-error increases from $0.0$ to $0.1$, the total number of feasible models increases dramatically from $3,640$ to $199,314$. While the specific values of these numbers are a function of our discretization and the value of $N$ used in the constraint program, they still enable us to make relative comparisons about the size of the fairness region. For example, Figure \ref{fig:varying_e_p} (c), where $\epsilon \leq 0.05$ (\ie the maximum allowable difference between group metrics is $5\%$), provides a valuable insight: if the prevalence difference between groups is less than $0.2$ (or $20\%$), the fairness region is quite dense, especially relative to plot \ref{fig:varying_e_p} (a), where $\epsilon = 0.0$. \textit{This is good news for practitioners} because: (1) prevalence differences between $10\%$ and $15\%$ are commonly observed, and (2) setting $\epsilon \leq 0.05$ is reasonable in many contexts.


\begin{figure}
\centering
\subfloat[]
{\includegraphics[width=.35\textwidth]{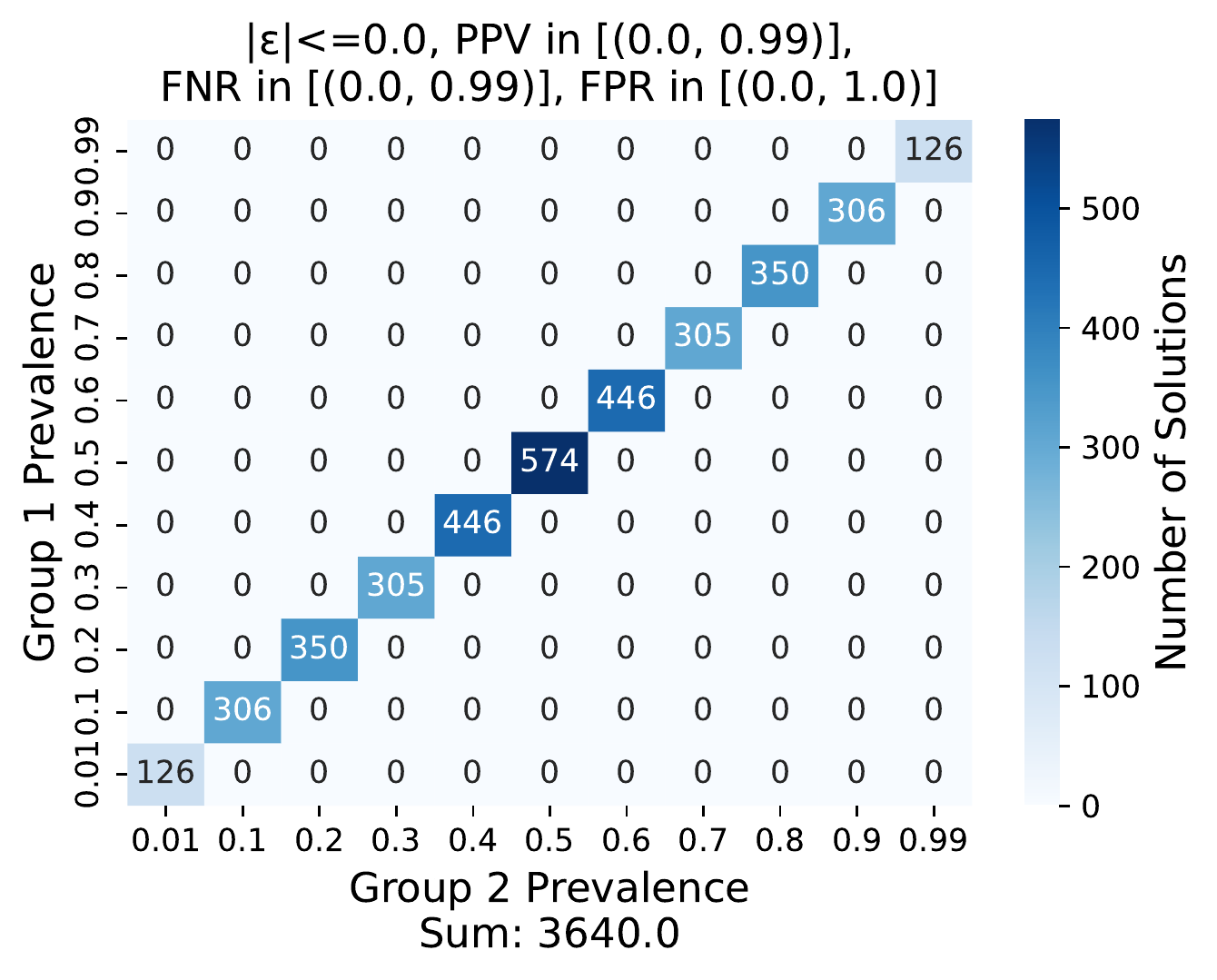}}
\subfloat[]
{\includegraphics[width=.35\textwidth]{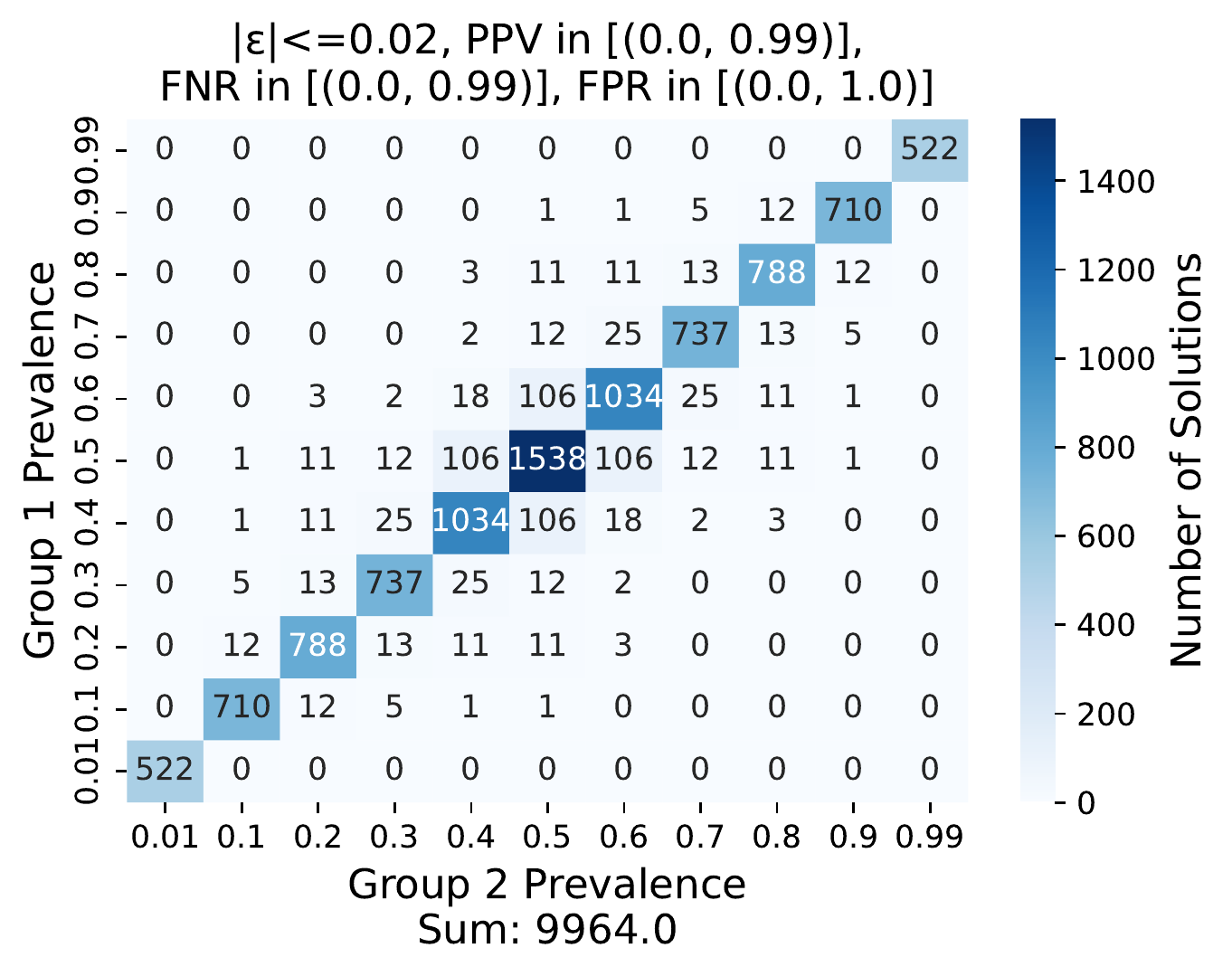}
}
\vspace*{0.01cm}
\subfloat[]
{\includegraphics[width=.35\textwidth]{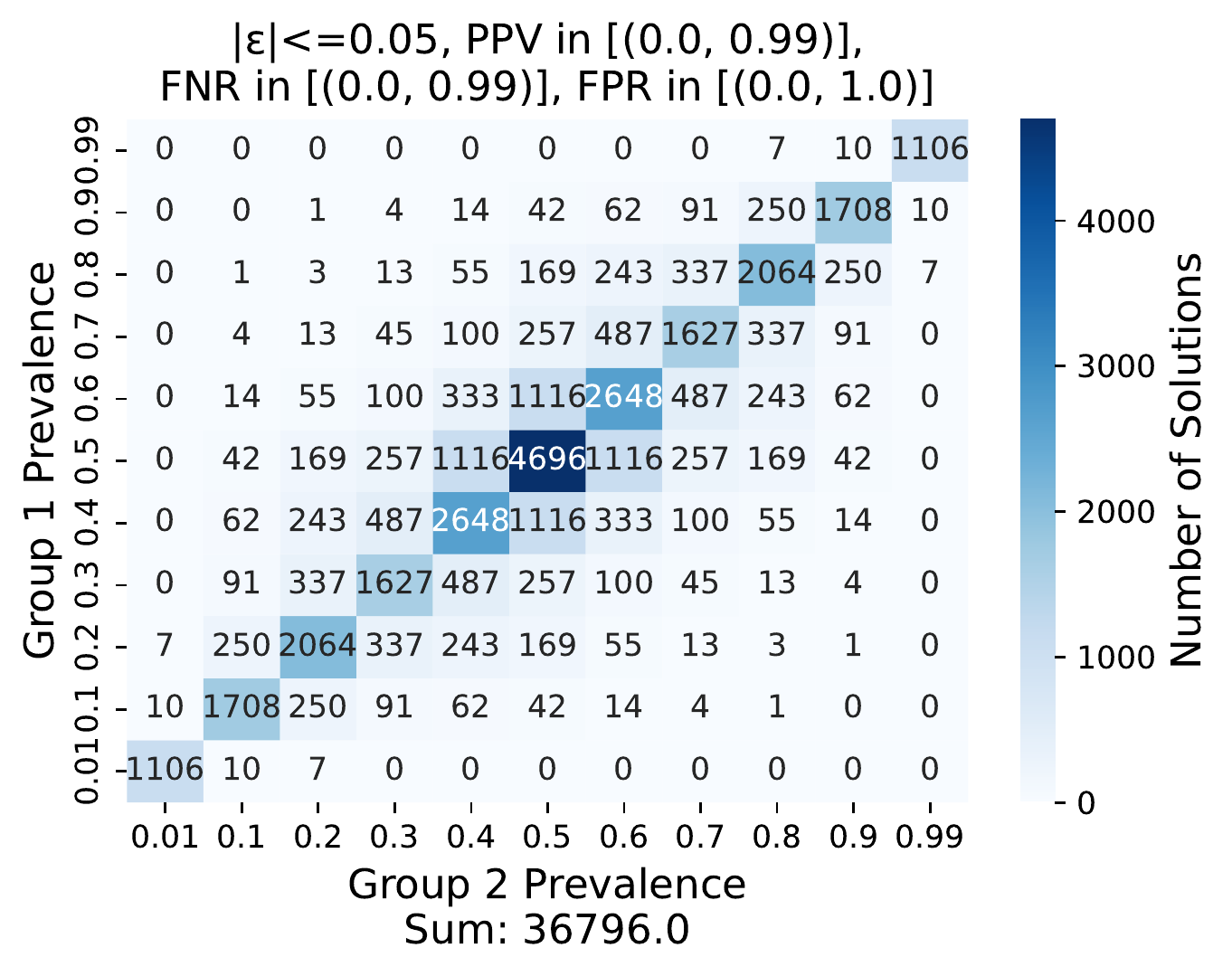}
}
\subfloat[]
{\includegraphics[width=.35\textwidth]{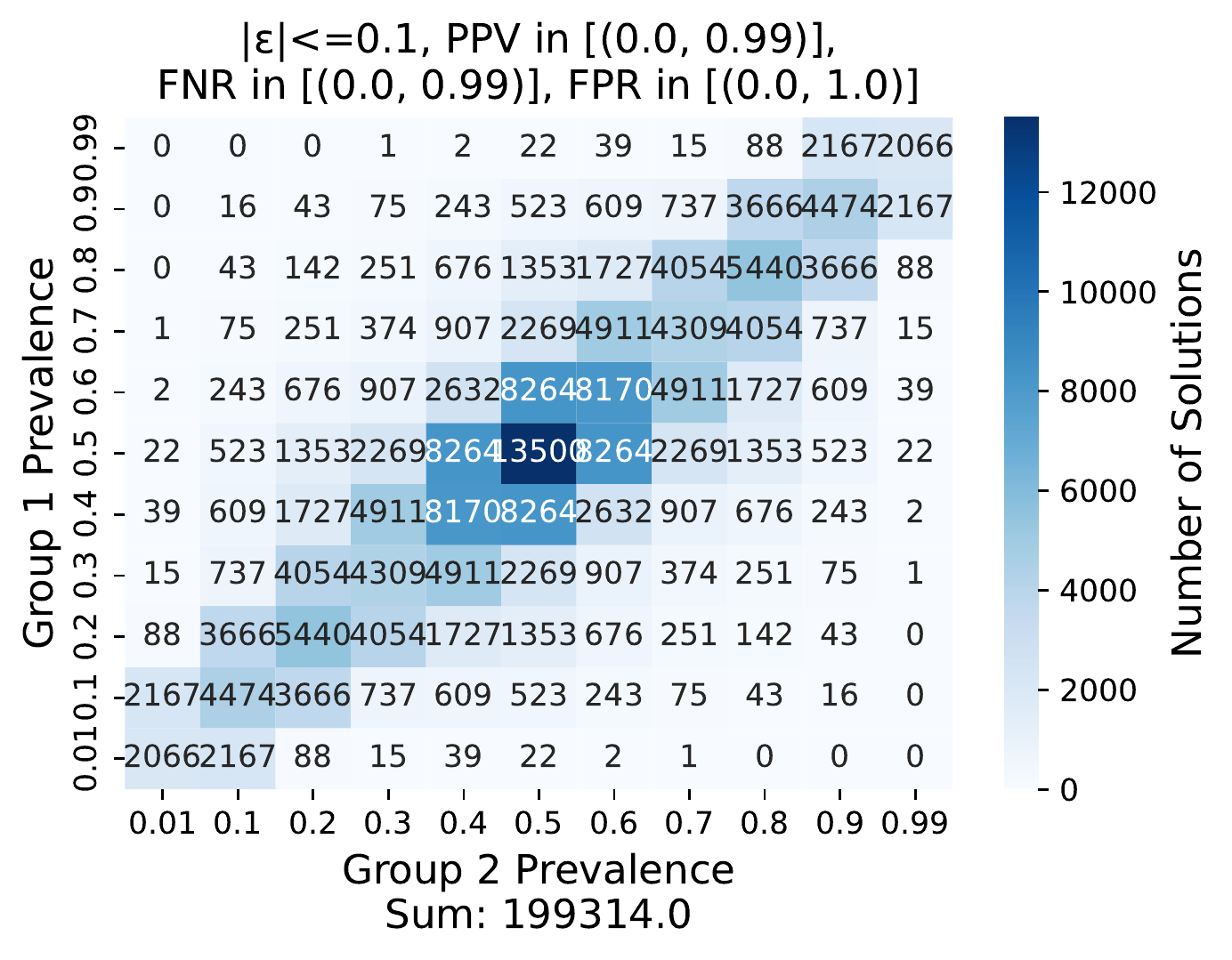}
}
\vspace{-0.3cm}
\caption{Effect of varying group prevalences $\prev_1, \prev_2$ on the number of feasible models for different values of $\epsilon$, where $\ppv, \fnr \in \lbrack 0, 0.99 \rbrack, \fpr \in \lbrack 0, 1.0 \rbrack$, $N=100$.}
\label{fig:varying_e_p}
\vspace{-0.5cm}
\end{figure}

\subsubsection{Varying performance}
\label{sec:fix_e_p_vary_ppv}

Next, we test the effect of ``imperfect prediction'' on the size of the fairness region. 
As a reference point, we focus on the case where $\epsilon \leq 0.05$ (Figure~\ref{fig:varying_e_p} (c)), and create bins that corresponded to four ranges of \ppv: $\lbrack 0.00, 0.24\rbrack$, $\lbrack 0.25,0.49 \rbrack$, $\lbrack 0.50,0.74 \rbrack$, and $\lbrack 0.75, 0.99 \rbrack$. 
As expected, the closer the setting is to ``perfect prediction'' (\ie the higher the \ppv), the larger the size of the fairness region. The number of feasible models increases from $7,554$ in the lowest \ppv bin to $10,007$ in the highest bin. Notably, there is still a large number of feasible models available in all bins even when the prevalence difference between groups is as high as 20\%.

There is another key insight implicit in Figure~\ref{fig:varying_ppv}: it not only shows how many feasible models there are under different \ppv settings, but that many of those models are \emph{high-performing}. In each figure, it can be seen that the number of feasible models is most dense when the group prevalences are below the maximum \ppv value. Recall that any model with a \ppv greater than the overall prevalence of the dataset on which it is being used offers value over random chance, suggesting not only many possible models, but many \emph{useful} models in these settings.

\begin{figure}
\centering
\subfloat[]
{\includegraphics[width=.35\textwidth]{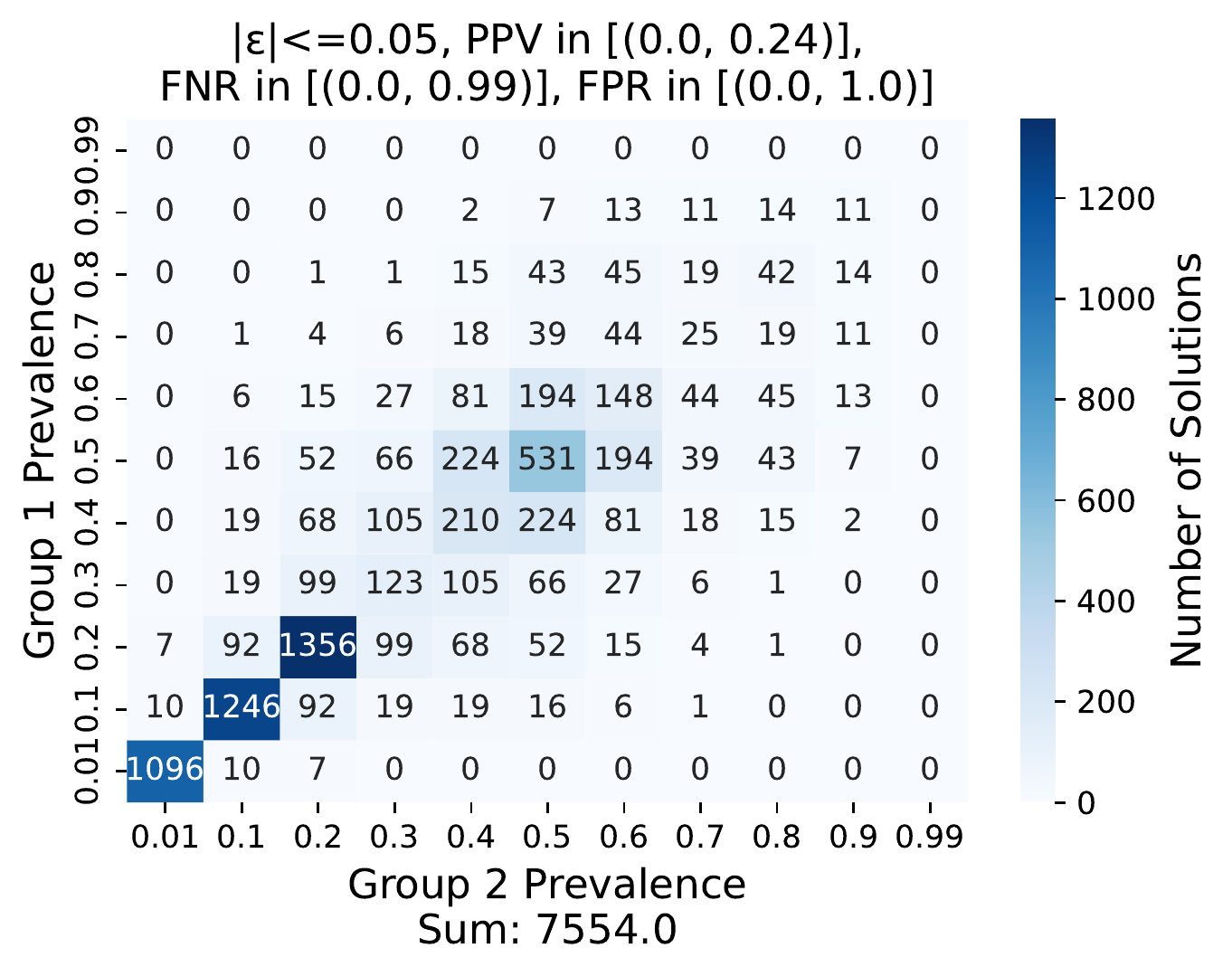}}
\subfloat[]
{\includegraphics[width=.35\textwidth]{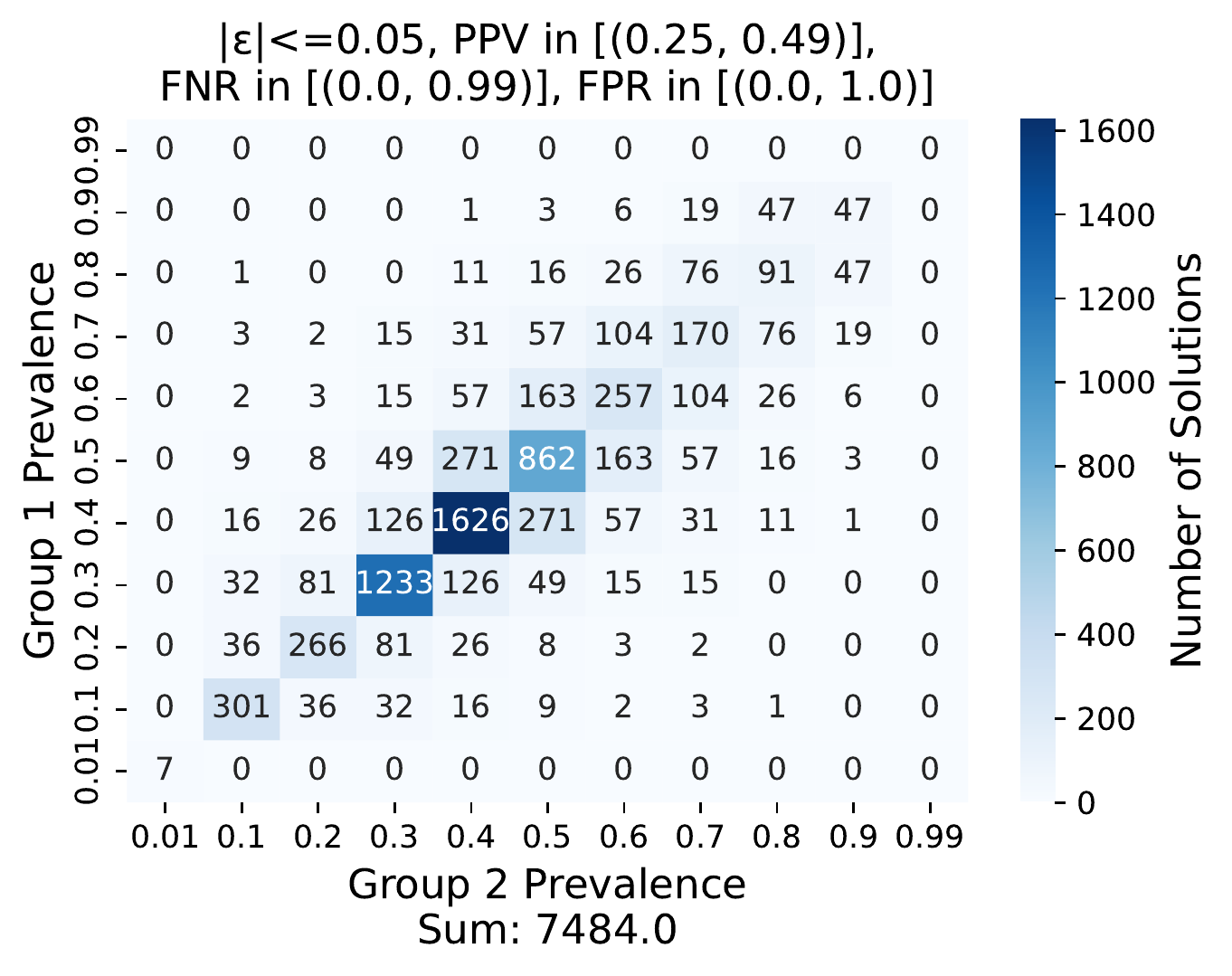}}
\vspace*{0.01cm}
\subfloat[]
{\includegraphics[width=.35\textwidth]{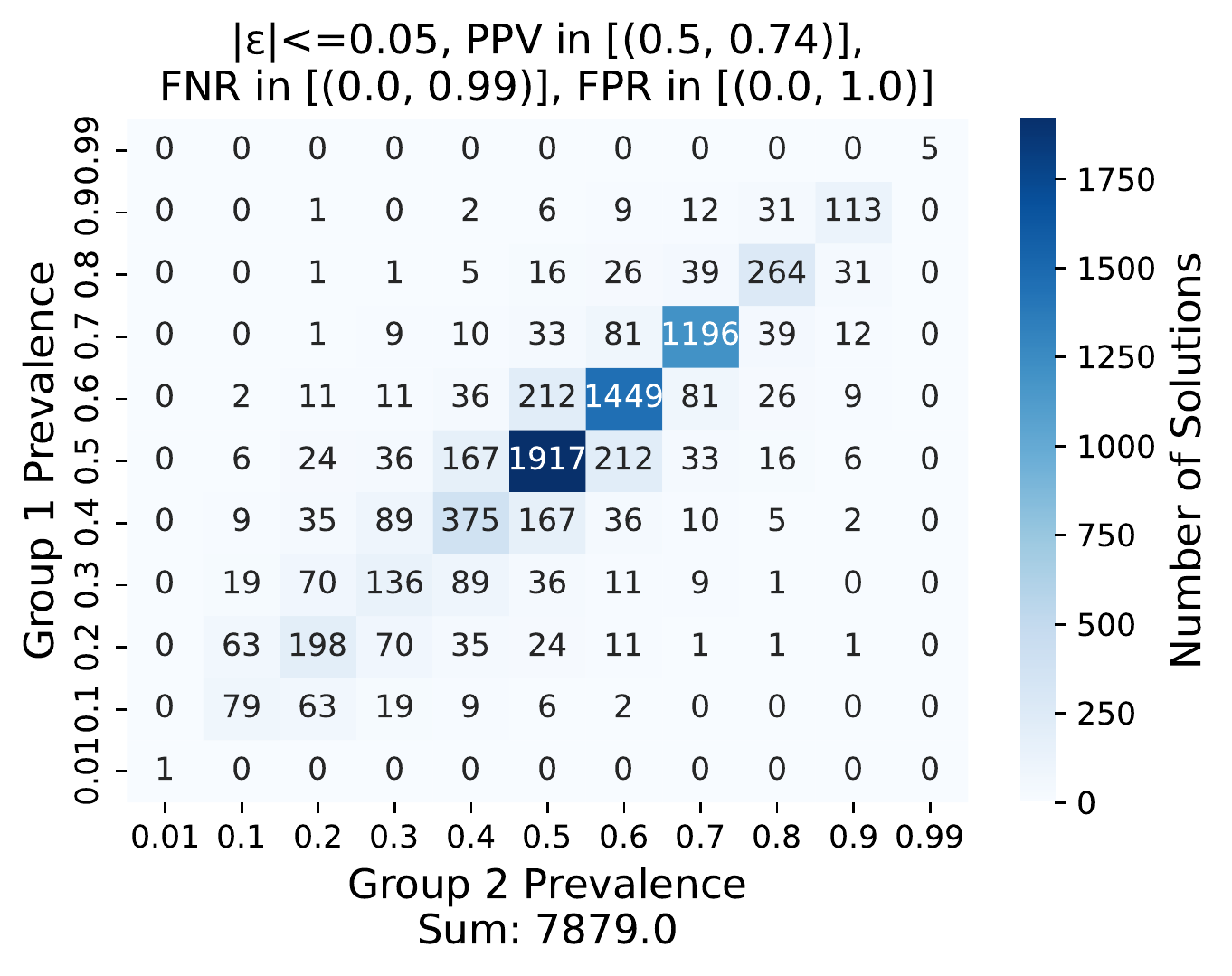}}
\subfloat[]
{\includegraphics[width=.35\textwidth]
{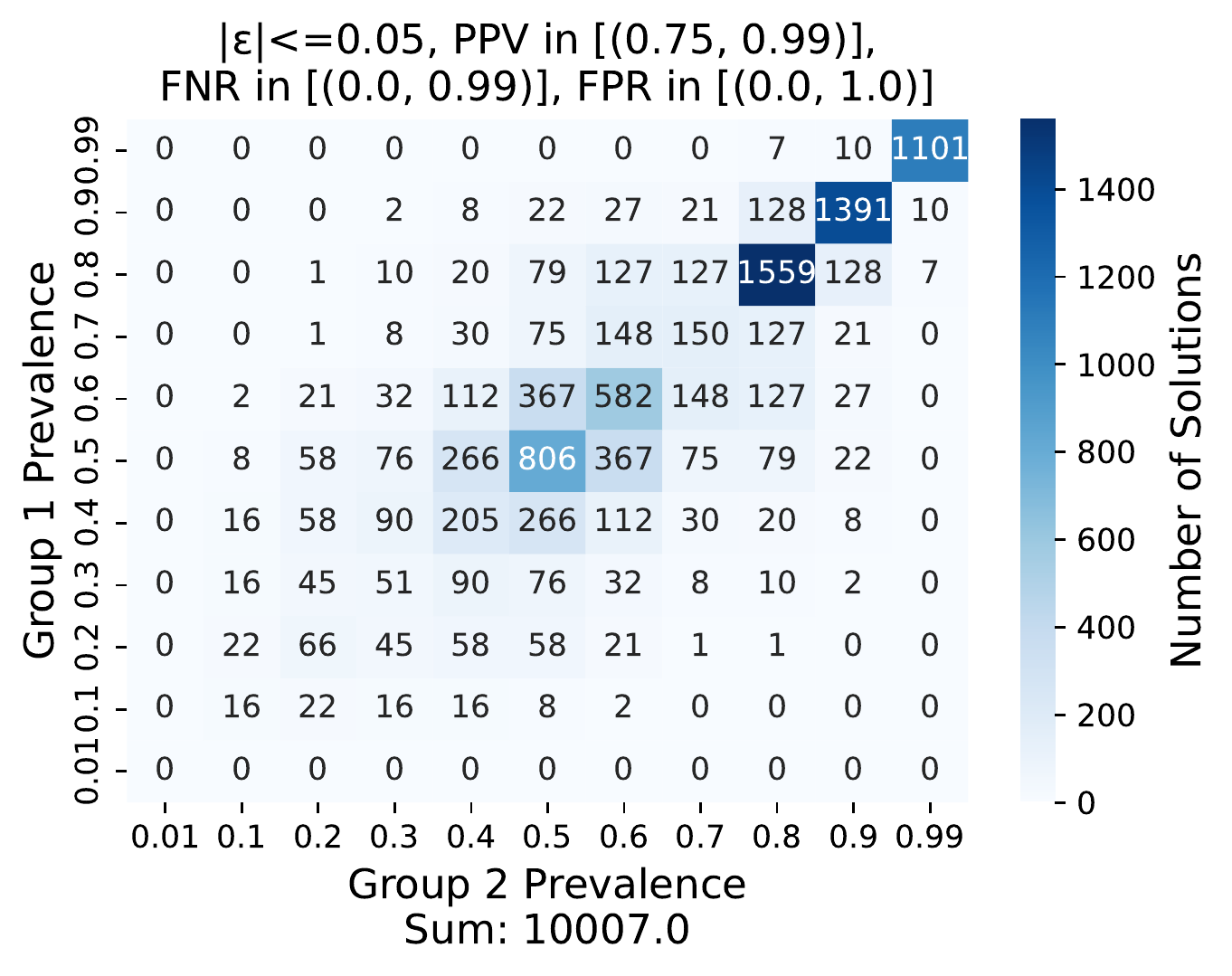}}
\vspace{-0.3cm}
\caption{Effect of varying \ppv on the fairness region, where $|\epsilon_\alpha|, |\epsilon_\beta|, |\epsilon_v| \leq 0.05$, $N=100$.}
\label{fig:varying_ppv}
\vspace{-0.5cm}
\end{figure}

\subsubsection{Considering intersectional groups}
\label{sec:intersectionality}

The ultimate goal of algorithmic fairness should not just be ensuring fairness according to multiple metrics, but also for multiple groups (\eg defined both based on sex, and based on race), including also for intersections of these groups (\eg defined by a combination of sex and race)~\cite{hankivsky2022intersectionality}.
Intersectional discrimination~\cite{crenshaw1990mapping, makkonen2002multiple} states that individuals who belong to several protected groups simultaneously (\eg Black women) experience stronger discrimination compared to individuals who belong to a single protected group (\eg White women or Black men), and that this disadvantage compounds more than additively.  This effect has been demonstrated by numerous case studies, and by theoretical and empirical work~\cite{collins2002black,shields2008gender,d2020data,noble2018algorithms}. 

Intersectionaliy is an analytical framework for understanding human beings that considers the outcome of intersections of different social locations, power relations and experiences~\cite{hankivsky2022intersectionality}. For example, an intersectional approach to fairness could be thinking beyond an individual's sex or race, and instead accounting for a set of important characteristics about that individual like their sex, race, ethnicity, and social class.  In this paper, we consider a limited interpretation of intersectionality, and investigate how stating fairness constraints with respect to intersections of several sensitive attributes impacts the existence of feasible models.  In the following proposition, we show that the maximum prevalence difference across groups defined by an intersection of sensitive attributes (\eg on sex and race) is at least as high as when groups are defined based on each sensitive attribute independently (\eg on sex or race). 



\begin{proposition}[Intersectional Prevalence Differences]
\label{theorem:int_prev_diff}
Given a dataset that is subdivided into two groups, let $0 < p_1 < n_1$ and $0 < p_2 < n_2$, where $p_i$ is the number of positive class members of group $i$, and $n_i$ is the total number of members in group $i$. Suppose $\frac{p_1}{n_1} \leq \frac{p_2}{n_2}$. Then the following holds: $\frac{p_1}{n_1} \leq \frac{p_1 + p_2}{n_1 + n_2} \leq \frac{p_2}{n_2}$. (Proof deferred to Appendix~\ref{appendix:proofs}).  
\hspace*{0pt}\hfill \qedsymbol
\end{proposition}

Consider a toy example where there are two binary sensitive attributes: \emph{sex} coded as male and female, and \emph{race} coded as majority and minority. Under the mild assumption of Proposition~\ref{theorem:int_prev_diff}, the prevalence of the minority group as a whole must be between the prevalence of the intersectional minority male and minority female groups. The same is true for the prevalence of the majority group. As a result, the prevalence difference between the four intersectional groups (majority male, majority female, minority male, minority female) \emph{must be greater than or equal} to the prevalence difference between only the majority and the minority group. The same reasoning can be used to understand prevalence differences for sex and intersectional sex. The overall implication of Proposition~\ref{theorem:int_prev_diff} is that considering intersectional groups leads to at least equal, but more commonly greater, prevalence differences between groups, which suggests there will be fewer feasible models that are fair with respect to \fnr, \fpr, and \ppv (\ie a smaller fairness region). 

\subsubsection{Varying the resource constraint \rc}
We now investigate the impact of \rc on our ability to identify a feasible solution.

\begin{proposition}[Reducing \rc increases ppv]
\label{theorem:role_of_k}
Given a \emph{well-calibrated classifier} being used under a resource constraint \rc, reducing the size of \rc will monotonically increase the \ppv of the classifier.
(Proof deferred to Appendix~\ref{appendix:proofs})
\hspace*{0pt}\hfill \qedsymbol
\end{proposition}

The insight of Proposition~\ref{theorem:role_of_k} is that reducing \rc causes a chain reaction: first, it increases the \ppv of the classifier, and second, an increase in \ppv results in a more dense space of feasible solutions (as observed in Section~\ref{sec:fix_e_p_vary_ppv}). Taken together, this suggests that reducing \rc can result in a denser space of feasible models  on \fpr, \fnr and \ppv.  
\section{Experiments}
\label{sec:experiments}


The analysis in Section~\ref{sec:theory} shows that, by slightly relaxing fairness constraints between metrics, there are a large number of models satisfying approximate fairness constraints across multiple metrics. In this section, we design an experiment to demonstrate the existence of those models\textit{ on real data}.


Our insights suggest that the possibility of finding fair models is influenced by (1) the group prevalences' proximity to $50\%$, (2) the differences between group prevalences, and (3) the performance of the classifier. To better understand how these parameters impact one's ability to find fair models on real-world data, we developed an experiment to answer the following question: Given a dataset $X$, a resource constraint \rc, a set of fairness constraints, and a classifier with a given \ppv, does there exist a set of \rc observations for which (1) fairness constraints for \fpr, \fnr and \ppv are satisfied, and (2) those fairness constraints do not reduce the \ppv? If there does exist a set of observations in $X$ that satisfies these requirements, then there also exists a model that could select those observations. Trivially, one can think of a function that uses the index of each element to map to an outcome. In other words, one could use such a set as the labels $Y$ for creating a function $i : X \rightarrow Y$.

To implement the experiment we created a Mixed Integer Linear Program as follows: The objective function is to maximize the \ppv of a selection of \rc observations, subject to 5 constraints: (1) \rc observations must be selected, (2) the classifier has \emph{at most} a pre-defined \ppv, (3-5) fairness constraints for \fnr, \fpr, and \ppv are met. The full program details can be found in Appendix~\ref{appendix:mip}.~\footnote{All data, code, and experimental results are available in a GitHub repository at \url{https://anonymous.4open.science/r/fairness-trade-offs-facct-submission-D2A4/README.md}.  The repository will be made public upon publication.} 

One strength of this experiment is that it allows us to work with datasets that may have non-binary sensitive attributes or multiple sensitive attributes. This means we can also explore the feasibility of finding fair sets \textit{when there are more than two groups and consider intersectionality} (see Section~\ref{sec:intersectionality}).

Note that there are two complications in our problem: First, because \ppv cannot be written as a linear constraint~\cite{hsu2022pushing}, we used an existing method to refactor our problem into an approximately equivalent Quadratic Linear Program that includes constraints for \ppv. Second, as a result of the transformations and the inherent complexity, running times become intractable for large datasets. To circumvent this problem, we conducted our experiments on a sample of each dataset stratified on the outcome and sensitive attributes. 
We ran several sensitivity analyses and did not find any meaningful difference in the experimental results due to down-sampling, but we acknowledge this as a limitation. 


\subsection{Datasets}

We worked with 5 real-world datasets, all with varying outcome prevalences, and representing a variety of sensitive attributes, including sex, race, and education level (which may also be a proxy for income). We used these datasets in scope of 16 tasks (\ie outcomes): 8 for \textit{Ukrainian EIE}, 5 for \textit{folktables}, 1 for each of the other 3 datasets).
We describe these datasets briefly here, see Appendix~\ref{sec:full_dataset_descriptions} for details.

\begin{itemize}[leftmargin=*]
    \item {\bf Ukranian External Independent Evaluation (EIE)\footnote{\url{https://zno.testportal.com.ua/opendata}}.} EIE data contains standardized tests for secondary school graduates in Ukraine. The 2021 data contains 389,322 records; sensitive attributes include the students' sex and whether they live in an urban or a rural area. Outcomes are students' performance on 5 tests (\eg \emph{history}, \emph{German}).
    \item {\bf Portuguese Student Performance
    ~\cite{Cortez2008UsingDM}.} This dataset contains the performance of Portuguese students from two high schools. The dataset was collected in 2006, contains 1,044 records, and includes  sensitive attributes like the student's sex, their parents' education levels, and whether the students live in an urban or rural location.
    \item {\bf Taiwanese Loan Assessment    
    } This dataset has customer loan data from a bank and cash issuer in Taiwan. The data was collected in 2005, has 30,000 records, and includes sensitive attributes like sex and education level. The associated task is to identify customers at risk of defaulting on their loan payments.
    \item {\bf Bangladeshi Diabetes Risk Assessment    
    ~\cite{islam2020likelihood}.} This dataset, published in 2020, has 520 patient records with information on diabetes-related symptoms, obtained through a questionnaire by the Sylhet Diabetes Hospital in Bangladesh. The task is to identify individuals at risk of early-stage diabetes. Sensitive attributes include age and sex.  
    \item {\bf Folktables
    ~\cite{ding2021retiring}.} This data is from the American Community Survey, and contains individual-level data related to income, employment, health, transportation, and housing. The data has millions of records total, and sensitive attributes include race and sex. We look at 5 separate pre-defined prediction tasks, using data from New York in 2018.
\end{itemize}

\subsection{Results}
\label{sec:exp_results}

Full experimental results are reported in Table~\ref{tab:exp_results} in Appendix~\ref{app:exp_results}, and truncated results are in Table~\ref{tab:exp_results_trunc}. These tables show the ``Optimal \rc Range'' for each \emph{(dataset (outcome), sensitive attribute) pair}, where \rc is expressed as a percentage of the observations. The optimal range shows for which values of \rc there is a set of observations that satisfy fairness constraints for \fpr, \fnr, and \ppv, without sacrificing any additional classifier \ppv. 
Note that in Table~\ref{tab:exp_results}, each \rc list must contain 30\% errors (\ie false positives). This means that the precision of the list \rc is at-most 70\%. This was an arbitrary choice;  however, we did conduct extensive sensitivity analysis (see Table~\ref{tab:exp_results_ppv_sa} in the Appendix) and found that increasing the  \ppv generally increases the Optimal \rc Range, as suggested in Section~\ref{sec:fix_e_p_vary_ppv}. 
Note also that ``Maximum Group Difference'' refers to the maximum pairwise difference between group prevalence values.

Recall our discussion about fairness constraints over intersectional groups in Section~\ref{sec:intersectionality}. In Table~\ref{tab:exp_intersect_results}, we have included results for \textit{Folktables} for a new ``race-sex'' attribute that partitions the data based on a combination of values of these two attributes. 
It can be seen that the same observations made in Section~\ref{sec:exp_results} hold here, but with one key difference: in general, the maximum prevalence difference in groups defined by an intersection of attributes (\eg on sex and race) is at least as high as when groups are defined based on each sensitive attribute independently (\eg on sex or race).  This is consistent with the insights of Proposition~\ref{theorem:int_prev_diff} in Section~\ref{sec:intersectionality}.

The results in Table~\ref{tab:exp_results} support the conceptual findings presented in Section~\ref{sec:theory}. The size of the optimal \rc range mirrors some previous findings, with evidence of larger \rc ranges for prevalences closer to $50\%$ and for smaller group prevalence differences.
As evidence for the former, for any row in Table~\ref{tab:exp_results} for which the optimal \rc range is \emph{All}, the group prevalence differences are small (less than $10\%$). For the latter, consider this interesting observation from the table: for the \emph{(EIE (German), Territory)} pair, the maximum group prevalence difference is $18.4\%$ --- yet there is no value of \rc where it is possible to simultaneously satisfy all three fairness constraints. Notice that the overall prevalence is around $10.35\%$. In contrast, the \emph{(Folktables (Travel Time), Race)} pair has a maximum prevalence difference that is even higher at 20.64\%, but in this case the group prevalences are closer to $50\%$, and the optimal \rc range spans over half the dataset ($\lbrack 5,55 \rbrack$).

Another salient result from our experiments is that out of all 32 combinations of \emph{(dataset (outcome), sensitive attribute)}, only 3 have no \rc value for which it was possible to find a set of observations satisfying every fairness constraint. This is a promising result: \textbf{across five separate and diverse real-world datasets}, we demonstrated that there is \textbf{nearly always} at least \emph{some} chance of finding \textbf{a model that is simultaneously fair with respect to \fpr, \fnr, and \ppv} with a small margin-of-error.


\begin{table}
\caption{Truncated experimental results}
\label{tab:exp_results_trunc}
\small
\begin{tabular}
{|p{3cm}H|p{1.5cm}|p{1.2cm}H|p{3.7cm}|p{1.8cm}|p{1cm}|}
\hline
\textbf{Dataset (Outcome)} & \textbf{Number of Samples} & \textbf{Overall Prevelance (\%)} & \textbf{Sensitive Attribute} & \textbf{Group Distribution (\%)} & \textbf{Group Prevelance (\%)} & \textbf{Maximum Prevelance Difference (\%)} & \textbf{Optimal \rc Range\textsuperscript{1}} \\ \hline
EIE (Ukranian) & 1445 & 7.34 & Sex & Female: 51.97; Male: 48.03 & Female: 4.26; Male: 10.66 & 6.4 & None \\ \hline
EIE (Ukranian) & 1445 & 7.34 & Territory & Rural: 25.33; Urban: 74.67 & Rural: 10.38; Urban: 6.3 & 4.08 & {[}5,20{]} \\ \hline
EIE (Math) & 1222 & 31.1 & Sex & Female: 48.69; Male: 51.31 & Female: 30.92; Male: 31.26 & 0.34 & All \\ \hline
EIE (Math) & 1222 & 31.1 & Territory & Rural: 25.2; Urban: 74.8 & Rural: 39.94; Urban: 28.12 & 11.82 & {[}5,90{]} \\ \hline
EIE (Geography) & 1132 & 5.3 & Sex & Female: 48.23; Male: 51.77 & Female: 4.21; Male: 6.31 & 2.1 & All \\ \hline
EIE (Geography) & 1132 & 5.3 & Territory & Rural: 27.47; Urban: 72.53 & Rural: 6.43; Urban: 4.87 & 1.56 & All \\ \hline
EIE (German) & 1004 & 11.35 & Sex & Female: 65.54; Male: 34.46 & Female: 10.03; Male: 13.87 & 3.84 & {[}5,90{]} \\ \hline
{\bf EIE (German)} & 1004 & \textbf{11.35} & {\bf Territory} & Rural: 15.14; Urban: 84.86 & Rural: 26.97; Urban: 8.57 & \textbf{18.4} & \textbf{None} \\ \hline
Folktables (Employment) & 1970 & 46.45 & Sex & Female: 51.68; Male: 48.32 & Female: 44.3; Male: 48.74 & 4.44 & All \\ \hline
Folktables (Employment) & 1970 & 46.45 & Race & Asian alone: 8.63; Black or African American alone: 12.18; Other: 8.88;   White: 70.3 & Asian alone: 50.0; Black or African American alone: 42.08; Other: 41.71;   White: 47.36 & 8.29 & All \\ \hline
Folktables (Travel Time) & 1824 & 53.78 & Sex & Female: 49.4; Male: 50.6 & Female: 51.72; Male: 55.8 & 4.08 & All \\ \hline
{\bf Folktables~(Travel~Time)} & 1824 & \textbf{53.78} & {\bf Race} & Asian alone: 9.27; Black or African American alone: 11.02; Other: 8.06;   White: 71.66 & Asian alone: 66.86; Black or African American alone: 69.15; Other: 64.63;   White: 48.51 & \textbf{20.64} & \textbf{{[}5,55{]}} \\ \hline
Loan Assessment & 1200 & 22.17 & Sex & Female: 60.33; Male: 39.67 & Female: 20.86; Male: 24.16 & 3.3 & All \\ \hline
Diabetes Risk Assessment & 520 & 61.54 & Sex & Female: 36.92; Male: 63.08 & Female: 90.1; Male: 44.82 & 45.28 & {[}5,10{]} \\ \hline
Student Performance & 1003 & 22.03 & Sex & Female: 56.63; Male: 43.37 & Female: 21.13; Male: 23.22 & 2.09 & All \\ \hline
Student Performance & 1003 & 22.03 & Parent's education level & High school: 24.13; Not high school or university or greater: 40.88;   University or greater: 35.0 & High school: 23.55; Not high school or university or greater: 25.85;   University or greater: 16.52 & 9.33 & {[}5,55{]} \\ \hline
\end{tabular}
\begin{flushleft}
\hspace{0.35cm}Notes: $^1$\rc is expressed as the percentage of the number of samples
\end{flushleft}
\vspace{-0.5cm}
\end{table}

\section{Discussion}
\label{sec:discussion}

This paper sought to revisit the impossibility theorem in practical settings. Our analytical and experimental results, taken together,  offer a promising perspective regarding the feasibility of finding models that are fair (under very slight relaxations) with respect to \fnr, \fpr, and \ppv. 
In this section, we present our findings as guidance for practitioners on when it will be feasible to find models that are fair with respect to multiple metrics. There are several considerations: (1) group prevalence values, (2) prevalence difference between groups, (3) classifier performance, and (4) resource constraint \rc. 

In the first two considerations, our findings suggest that if one allows a small margin-of-error difference between metrics, then there exist many models that simultaneously satisfying parity across \fnr, \fpr, \ppv even when there is a moderate prevalence difference between groups. Further exploration is needed to understand what exactly constitutes \emph{small} and \emph{moderate}, but in our analysis we observed that cases with a $5\%$ margin-of-error and prevalence differences up to $10\%$ (and in some cases up to $20\%$) afforded feasible solutions. We are unsure how well these particular settings will generalize, but the larger implication is hopeful. For example, revisiting the thought experiment from Section~\ref{sec:implications_of_impossibility_theorem}: when predicting student graduation in the US where the prevalence difference between minority and majority students is roughly $10\%$, 
we expect that \emph{it is possible to build models that are fair with respect to multiple metrics}. To allow practitioners to answer questions like these for their own datasets, we offer an open-source tool they can use to assess the feasibility of finding models that are fair across multiple constraints, given an input dataset.\footnote{\url{https://anonymous.4open.science/r/fairness-trade-offs-facct-submission-D2A4/README.md}}


Regarding the third and fourth considerations, our analytical work suggests that a higher \ppv yields a larger number of feasible models. This furthers claims by other researchers that increasing the performance of a model \emph{actually improves} the possibility of finding a fair model~\cite{wick2019unlocking}. This is in-line with a paradigm shift away from thinking one must choose between high performance \emph{or} fairness --- from a fairness perspective, it can be worthwhile to improve the performance of your classifier to further enable fairness across multiple constraints. Connecting this insight to the resource-constrained setting with \rc, it follows that \textbf{resource constraints can, perhaps counter-intuitively, lead to higher chances of finding fair models} (see Proposition~\ref{theorem:role_of_k}). This is particularly impactful for practitioners working in ML for public policy, where resource constraints can be as small as $1\%$ $(\rc = 0.01)$ or $5\%$ $(\rc = 0.05)$~\cite{rodolfa2021empirical, bell2022s}.

We also offer two other meta-considerations. The first is the $\epsilon$-margin-of-error allowed between fairness metrics. In practice, $\epsilon$ should be decided \emph{a priori}, and by consulting stakeholders and subject area experts~\cite{ruf2021towards, saleiro2018aequitas}---but generally, the guidance here is unsurprising: the larger the tolerable difference between metrics, the larger the feasible region of fair models. 
The second meta-consideration is the number of groups of sensitive attributes. We find that adding intersectional groups will increase prevalence differences (see Proposition~\ref{theorem:int_prev_diff}), which reduces the number of possibilities for fair models.  However, this is by no means an argument against considering intersectionality. On the contrary, we frame this finding as follows: you can continue to add sensitive attributes and intersectional groups and still have a chance of finding models that are fair across multiple metrics.

\section{Conclusions and social impact}
\label{sec:conclusion}

In this paper we provide evidence that challenges the implications of the impossibility theorem in practical settings, suggesting that practitioners can strive for fairness with respect to multiple metrics simultaneously in the algorithms they implement. This is an important part of the social impact of this paper. It exists as part of a growing body of literature showing that strong limits to fairness, like trade-offs with performance, with other metrics, or between groups, may be over-stated or even self-imposed. The impossibility theorem is not a rigid barrier to equitable machine learning.

This work also further demonstrates the importance of reducing societal biases, which are ultimately what cause prevalence differences between groups to appear in data. There is a similar implication for designing better models, algorithms, and classifiers. By understanding how model performance and fairness are interrelated, we can shift away from a paradigm of wanting to build algorithms that are either better performing \emph{or} more fair, and towards one where we build algorithms that are better performing \emph{and} more fair.


 Our work leaves open an important next step in ensuring fairness across multiple metrics and for multiple groups: Once we know there is a large set of feasible models, how do we find such a model? Further, does having a large number of feasible models make it easier to find one of those models? Significant additional study of this problem should be done with closed-form expressions describing model feasibility in terms of \fpr, \fnr, and \ppv. 
 Notably, as of the time of this writing, there has been at least one effort made to develop an algorithm to find classifiers that are fair with respect to \fpr, \fnr, and \ppv~\cite{hsu2022pushing}. That effort, along with related work that challenges commonly-held beliefs about the fairness-accuracy trade-off, may represent an inflection point in the fair-ML community: fewer and fewer researchers and practitioners conform to the idea that we must choose between (a single notion of) fairness and accuracy~\cite{rodolfa2021empirical, celis2019classification, wick2019unlocking}. The main take-away of our work is that \textbf{achieving fairness along multiple metrics, for multiple groups, and without sacrificing accuracy is much more attainable than previously believed.}

\bibliographystyle{ACM-Reference-Format}
\bibliography{main}

\appendix
\newpage
\text{\LARGE Appendix.}
\section{Full Description of Datasets}
\label{sec:full_dataset_descriptions}

\begin{itemize}
    \item {\bf Ukrainian Center for Educational Quality Assessment\footnote{\url{https://zno.testportal.com.ua/opendata}}.} 
    This dataset contains the results of the External Independent Evaluation (EIE or ZNO) in 2021. EIE is a set of organizational procedures (primarily testing) aimed at determining the level of educational achievement of secondary school graduates upon their admission to higher education institutions in Ukraine. 
    The dataset contains 389,322 records from EIE participants in 2021. Features include demographic attributes (ex. birth, gender, region) and participant performance records in 13 subjects. Sensitive attributes include the students' sex and whether they live in an urban or rural location.

    \item {\bf Portuguese Student Performance\footnote{\url{https://archive.ics.uci.edu/ml/datasets/student+performance}}~\cite{Cortez2008UsingDM}.} This dataset contains the performance of Portuguese high school students aged approximately 15 to 19 for two subjects: mathematics and Portuguese language arts. The dataset contains 1,044 records from students at two high schools (one urban and one rural) and was collected in 2005 and 2006. The associated prediction task is identifying students at risk of failure to provide additional school resources. Features include administrative records from schools (ex. grades, number of absences) and a lifestyle questionnaire completed by each student (ex. how many hours per week they study). Sensitive attributes include the student's sex, their parent's education level (which may be a proxy for income), and whether or not the students live in an urban or rural location.
    \item {\bf Taiwanese Loan Assessment\footnote{\url{https://archive.ics.uci.edu/ml/datasets/default+of+credit+card+clients}}~\cite{yeh2009comparisons}.} This dataset contains customer loan data from an important bank (and cash issuer) in Taiwan. The data has 30,000 records, and includes features like loan applicant's age, education status, marital status, and payment history. It was gathered in 2005. The machine learning task associated with this dataset is identifying customers at risk of defaulting on their loan payments. Sensitive attributes include sex and education status.
    \item {\bf Bangladeshi Diabetes Risk Assessment\footnote{\url{https://archive.ics.uci.edu/ml/datasets/Early+stage+diabetes+risk+prediction+dataset.}}~\cite{islam2020likelihood}.} This dataset has 520 patient records and contains information on diabetes-related symptoms, obtained through a questionnaire carried out by the Sylhet Diabetes Hospital in Sylhet, Bangladesh. The prediction task associated with the dataset is identifying individuals at risk of developing early stage data, and the academic work accompanying the data was published in 2020. Sensitive attributes include age and sex.  
    \item {\bf Folktables\footnote{\url{https://github.com/zykls/folktables}}~\cite{ding2021retiring}.} This data is from the American Community Survey Public Use Microdata Samples (ACM PUMS), and contains individual- and household-level data related to income, employment, health, transportation, and housing in the United States. The data is updated yearly, is available at both the national or state level, and contains millions of records. Sensitive attributes include race and sex. We look at 5 separate pre-defined prediction tasks, using data from Massachusetts in 2018 (documentation can be found in the footnoted GitHub repository).
\end{itemize}

\section{Quadratic Mixed Integer Linear Program Details}
\label{appendix:mip}

In this section, we present the Mixed Integer Linear Program used in our experiment in Section~\ref{sec:experiments}.  The objective function is to maximize the \ppv of a selection of \rc observations, subject to 5 constraints: (1) \rc observations must be selected, (2) the classifier has \emph{at most} a pre-defined \ppv, (3-5) fairness constraints for \fnr, \fpr, and \ppv are met. All data, code, and experimental results are available in a GitHub repository at \url{https://anonymous.4open.science/r/fairness-trade-offs-facct-submission-D2A4/README.md}.  The repository will be made public upon publication. 

Note that the fairness constraints are enforced using \emph{disparity ratios} where $0.8 \leq disparity \leq 1.2$, rather than a $\pm \epsilon$ distance between group metrics. We made this decision because disparity ratios are more robust to small real values. For example, consider a model that has $\fpr_1 = 0.002$ for one group and $\fpr_2 = 0.04$ for the other. If $\epsilon=0.05$, technically these values would satisfy a fairness constraint where $|\fpr_1 - \fpr_2| <= \epsilon$, but this is likely not desirable to practitioners. However, using disparity ratios, this scenario would be considered unfair since $0.8 \nleq \frac{\fpr_1}{\fpr_2} \leq 1.2$.

We begin with the following Mixed Integer Linear Program:

\begin{equation*}
\begin{array}{llll}
\text{maximize}_{\displaystyle x_i}  & \displaystyle\sum\limits_{i=1}^{n}  x_i \cdot l_i & &\\

\text{subject to} & \displaystyle\sum\limits_{i=1}^{n}  x_i = k,& & \text{(KLS)}\\

& \frac{\displaystyle\sum\limits_{i=1}^{n}  x_i\cdot(1-l_i) \cdot g_{j_i}}{\displaystyle\sum\limits_{i=1}^{n}(1-l_i) \cdot g_{j_i}} \leq ub\cdot \frac{\displaystyle\sum\limits_{i=1}^{n} x_i\cdot(1-l_i) \cdot g_{ref_i}}{\displaystyle\sum\limits_{i=1}^{n} (1-l_i) \cdot g_{ref_i}},  & \forall j \in G, j \ne ref & \text{(FPRU)}\\

& \frac{\displaystyle\sum\limits_{i=1}^{n}  x_i\cdot(1-l_i) \cdot g_{j_i}}{\displaystyle\sum\limits_{i=1}^{n}(1-l_i) \cdot g_{j_i}} \geq lb\cdot \frac{\displaystyle\sum\limits_{i=1}^{n} x_i\cdot(1-l_i) \cdot g_{ref_i}}{\displaystyle\sum\limits_{i=1}^{n} (1-l_i) \cdot g_{ref_i}},  & \forall j \in G, j \ne ref & \text{(FPRL)}\\

& \frac{\displaystyle\sum\limits_{i=1}^{n}  (1-x_i)\cdot l_i \cdot g_{j_i}}{\displaystyle\sum\limits_{i=1}^{n} l_i \cdot g_{j_i}} \leq ub\cdot \frac{\displaystyle\sum\limits_{i=1}^{n} (1-x_i)\cdot l_i \cdot g_{ref_i}}{\displaystyle\sum\limits_{i=1}^{n} l_i \cdot g_{ref_i}},  & \forall j \in G, j \ne ref & \text{(FNRU)}\\

& \frac{\displaystyle\sum\limits_{i=1}^{n}  (1-x_i)\cdot l_i \cdot g_{j_i}}{\displaystyle\sum\limits_{i=1}^{n} l_i \cdot g_{j_i}} \geq lb\cdot \frac{\displaystyle\sum\limits_{i=1}^{n} (1-x_i)\cdot l_i \cdot g_{ref_i}}{\displaystyle\sum\limits_{i=1}^{n} l_i \cdot g_{ref_i}},  & \forall j \in G, j \ne ref & \text{(FNRL)}\\

& \frac{\displaystyle\sum\limits_{i=1}^{n}  x_i\cdot l_i \cdot g_{j_i}}{\displaystyle\sum\limits_{i=1}^{n} x_i \cdot g_{j_i}} \leq ub\cdot \frac{\displaystyle\sum\limits_{i=1}^{n} x_i\cdot l_i \cdot g_{ref_i}}{\displaystyle\sum\limits_{i=1}^{n} x_i \cdot g_{ref_i}},  & \forall j \in G, j \ne ref & \text{(PPVU)}\\

& \frac{\displaystyle\sum\limits_{i=1}^{n}  x_i\cdot l_i \cdot g_{j_i}}{\displaystyle\sum\limits_{i=1}^{n} x_i \cdot g_{j_i}} \geq lb\cdot \frac{\displaystyle\sum\limits_{i=1}^{n} x_i\cdot l_i \cdot g_{ref_i}}{\displaystyle\sum\limits_{i=1}^{n} x_i \cdot g_{ref_i}},  & \forall j \in G, j \ne ref & \text{(PPVL)}\\

& &x_{i}, l_{i} \in \{0,1\}, &i=1 ,\dots, n  \\
& &g_{j_i} \in \{0,1\}, & j \in G  \\

\end{array}
\end{equation*}

where $N$ is the set of entities, $n = |N|$; $x$ is a binary array of length $n$ where entry $x_i$ indicates whether or not entity $i \in N$ is included in the final list; $l$ is a binary array of length $n$ such that $l_i = 1$ if the outcome entity $x_i, i \in N$ is 1 and 0 otherwise; $G$ is a set of protected groups; $g_j$ is a binary array of length $n$ where entry $g_{j_i}$ indicates whether or not entry entity $i \in N$ is in the group $j$ (note that group $g_{ref}$ is the reference group for disparity calculations); $k$ is the final list size; $ub$ and $lb$ are the  upper and lower bounds for the disparity ratios, respectively.

KLS is the ``k-list-size'' constraint, FPRU and FPRL are the upper and lower bounds for the False Positive Rate, respectively, FNRU and FNRL are the upper and lower bounds for the False Negative Rate, respectively, and PPVU and PPVL are the upper and lower bounds for the \ppv respectively.

The \ppv constraints (PPVU and PPVL) halt the problem from being the Mixed Integer Linear Programming problem (MILP). We were inspired by an approach that was used when faced with a similar obstacle in creating MFOpt (Multiple Fairness Optimization Framework)~\cite{hsu2022pushing}, and propose a reformulation of the MIP in a way where we can apply the normalized multiparametric disaggregation technique (NMDT \cite{andrade2019nmdt}). We go through the following four steps:

{\bf Step 1.} Make the following substitution:

\begin{equation*}
    \displaystyle P_{g_j} =\sum\limits_{i=1}^{n}  x_i\cdot g_{j_i}\quad\text{and}\quad \displaystyle T_{g_j} = \frac{\displaystyle\sum\limits_{i=1}^{n}  x_i\cdot l_i\cdot g_{j_i}}{\displaystyle \sum\limits_{i=1}^{n}  x_i\cdot g_{j_i}}
\end{equation*}

Note that $P_{g_j}$ is the number of entities from the group $g_j$ that are included in the final list, and $T_{g_j}$ is the PPV for the group $g_j$.

{\bf Step 2.} Find the upper $PU_{g_j}$ and the lower $PL_{g_j}$ bounds for each $P_{g_j}$ by solving the following MILP:
\begin{equation*}
\begin{array}{ll}
\text{maximize}_{\displaystyle x_i}  & \displaystyle P_{g_j}\\

\text{subject to} &  \text{(KLS)}\\

&  \text{(FPRU)},\quad \text{(FPRL)}\\
&  \text{(FNRU)}, \quad \text{(FNRL)}\\

\end{array}
\end{equation*}

{\bf Step 3.} Find the upper $TU_{g_j}$ and the lower $TL_{g_j}$ bounds for each $T_{g_j}$ by solving the following MIP with linear constraints but with the fractional objective:
\begin{equation*}
\begin{array}{ll}
\text{maximize}_{\displaystyle x_i}  & \displaystyle T_{g_j}\\

\text{subject to} &  \text{(KLS)}\\

&  \text{(FPRU)},\quad \text{(FPRL)}\\
&  \text{(FNRU)}, \quad \text{(FNRL)}\\

\end{array}
\end{equation*}

To solve we use the Charnes-Cooper transformation \cite{charnes1962fractional}.

{\bf Step 4.} Reformulate the initial optimization problem in terms of $P_{g_j}$ and $T_{g_j}$ with corresponding lower and upper bounds from steps 2 and 3 in order to use NMDT transformation, so that it can be easily handled by MIP solver \cite{gurobi2022nmdt}.  Also, note that all denominators there are constant for the given dataset.

\begin{equation*}
\begin{array}{llll}
\text{maximize}  & \displaystyle\sum\limits_{j=1}^{n}  T_{g_j} \cdot P_{g_j} & &\\

\text{subject to} & \displaystyle\sum\limits_{j=1}^{n}  P_{g_j} = k,& & \text{(KLS)}\\

& \frac{\displaystyle  P_{g_j}\cdot(1-T_{g_j})}{\displaystyle\sum\limits_{i=1}^{n}(1-l_i) \cdot g_{j_i}} \leq ub\cdot \frac{\displaystyle P_{g_{ref}}\cdot(1-T_{g_{ref}})}{\displaystyle\sum\limits_{i=1}^{n} (1-l_i) \cdot g_{ref_i}},  & \forall j \in G, j \ne ref & \text{(FPRU)}\\

& \frac{\displaystyle  P_{g_j}\cdot(1-T_{g_j})}{\displaystyle\sum\limits_{i=1}^{n}(1-l_i) \cdot g_{j_i}} \geq lb\cdot \frac{\displaystyle P_{g_{ref}}\cdot(1-T_{g_{ref}})}{\displaystyle\sum\limits_{i=1}^{n} (1-l_i) \cdot g_{ref_i}},  & \forall j \in G, j \ne ref & \text{(FPRL)}\\

& 1- \frac{\displaystyle  P_{g_j}\cdot T_{g_j}}{\displaystyle\sum\limits_{i=1}^{n} l_i \cdot g_{j_i}} \leq ub \cdot \left(1- \frac{\displaystyle \displaystyle  P_{g_{ref}}\cdot T_{g_{ref}}}{\displaystyle\sum\limits_{i=1}^{n} l_i \cdot g_{ref_i}}\right),  & \forall j \in G, j \ne ref & \text{(FNRU)}\\

& 1- \frac{\displaystyle  P_{g_j}\cdot T_{g_j}}{\displaystyle\sum\limits_{i=1}^{n} l_i \cdot g_{j_i}} \geq lb \cdot \left(1- \frac{\displaystyle \displaystyle  P_{g_{ref}}\cdot T_{g_{ref}}}{\displaystyle\sum\limits_{i=1}^{n} l_i \cdot g_{ref_i}}\right),  & \forall j \in G, j \ne ref & \text{(FNRL)}\\

& \qquad\displaystyle T_{g_j} \leq ub\cdot T_{g_{ref}}  & \forall j \in G, j \ne ref & \text{(PPVU)}\\

& \qquad\displaystyle T_{g_j}\leq lb\cdot T_{g_{ref}}  & \forall j \in G, j \ne ref & \text{(PPVL)}\\

& &P_j \in \{PL_{g_j},\dots, PU_{g_j}\}, & j \in G  \\

& &T_j \in [TL_{g_j}, TU_{g_j}], & j \in G  \\

& &l_{i} \in \{0,1\}, &i=1 ,\dots, n  \\

& &g_{j_i} \in \{0,1\}, & j \in G  \\

\end{array}
\end{equation*}

Note that $P_{g_j}$ are integer variables, and $T_{g_j}$ are continuous  variables, but with bounds found in step 3 (
$T_{g_j} \in [TL_{g_j}, TU_{g_j}]$), and precision factor $p$ as a negative integer, we can represent this continuous variable exactly as 
\begin{equation*}
\begin{array}{ll}
  & \displaystyle T_{g_j} = (TU_{g_j}- TL_{g_j})\cdot \lambda + TL_{g_j}\\
 \text{where} &\\
 &  \displaystyle\lambda = \sum\limits_{m \in \{-p,\dots, -1\}}  2^m\cdot z_{m} \\

\end{array}
\end{equation*}
and $z_m \in \{0, 1\}$ are binary optimization variables.

\section{Experimental Results}
\label{app:exp_results}

Our full experimental results can be found in the Python notebooks in the ``experiments'' folder of our Github repository.\footnote{\url{https://anonymous.4open.science/r/fairness-trade-offs-facct-submission-D2A4/}} In this section, we wanted to include examples of the output of our experiment for each \emph{Dataset (Outcome), sensitive attribute} pair. Here we highlight two such pairs from the EIE dataset, which can be seen in Figure~\ref{fig:k_disp}. Plot (a) shows results for the \emph{EIE (Geography), territory} pair, and plot (b) shows the results for the \emph{EIE (Ukrainian), territory} pair. Each of those plots contains two subplots. On top, it shows the \ppv (Precision) and Recall (dotted lines) of an unconstrained linear program that has the specified \ppv. The solid lines show the \ppv and Recall of the selected sets. The bottom plot shows the disparity of each metric, where the dashed lines show the limits of $1.2$ and $0.8$. We can tell when a model is no longer optimal when the constrained \ppv and Recall meaningfully deviate from the unconstrained \ppv and Recall. Note that in some experiments, the \ppv, \fpr, and \fnr disparities may be outside of the disparity window ($\lbrack 0.8, 1.2 \rbrack$) --- but these instances are either pathological or due to a rounding error. The pathological cases occur when there is only one or two False Positives or False Negatives in a group.

\begin{figure}
\centering
\subfloat[]
{\includegraphics[width=.5\textwidth]{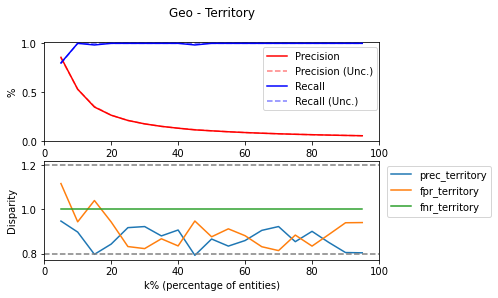}}
\subfloat[]
{\includegraphics[width=.5\textwidth]{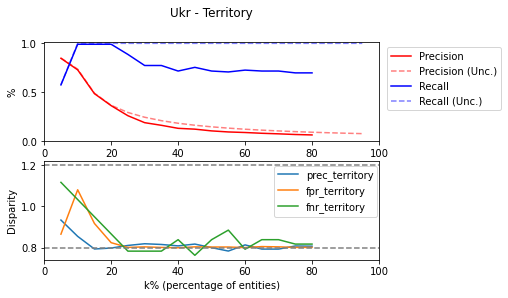}
}

\caption{Plot (a) shows that over all values of \rc, the \ppv and Recall of the selected set do not meaningfully deviate from an unconstrained model, and that the \ppv, \fpr, and \fnr remain within the bounds of the disparity window for all values; plot (b) shows that over the \rc range of $\lbrack 5,20 \rbrack$, the \ppv and Recall of the found sets do not deviate from the unconstrained Precision and Recall}
\label{fig:k_disp}
\end{figure}

\begin{itemize}
    \item Full experimental results: Table~\ref{tab:exp_results}
    \item Intersectional results: Table~\ref{tab:exp_intersect_results}
    \item Sample size sensitivity analysis: Table~\ref{tab:exp_results_sample_sa}
    \item \ppv sensitivity analysis: Table~\ref{tab:exp_results_ppv_sa}
\end{itemize}

\pagebreak

\begin{table}[H]
\caption{Emperical results}
\label{tab:exp_results}
\small
\begin{tabular}
{|p{3cm}H|p{1.5cm}|p{1.2cm}H|p{3.7cm}|p{1.8cm}|p{1cm}|}
\hline
\textbf{Dataset (Outcome)} & \textbf{Number of Samples} & \textbf{Overall Prevelance (\%)} & \textbf{Sensitive Attribute} & \textbf{Group Distribution (\%)} & \textbf{Group Prevelance (\%)} & \textbf{Maximum Prevelance Difference (\%)} & \textbf{Optimal \rc Range\textsuperscript{1}} \\ \hline
EIE (Ukranian) & 1445 & 7.34 & Sex & Female: 51.97; Male: 48.03 & Female: 4.26; Male: 10.66 & 6.4 & None \\ \hline
EIE (Ukranian) & 1445 & 7.34 & Territory & Rural: 25.33; Urban: 74.67 & Rural: 10.38; Urban: 6.3 & 4.08 & {[}5,20{]} \\ \hline
EIE (History) & 1994 & 18 & Sex & Female: 54.36; Male: 45.64 & Female: 14.48; Male: 22.2 & 7.72 & {[}5,60{]} \\ \hline
EIE (History) & 1994 & 18 & Territory & Rural: 29.69; Urban: 70.31 & Rural: 20.27; Urban: 17.05 & 3.22 & All \\ \hline
EIE (Math) & 1222 & 31.1 & Sex & Female: 48.69; Male: 51.31 & Female: 30.92; Male: 31.26 & 0.34 & All \\ \hline
EIE (Math) & 1222 & 31.1 & Territory & Rural: 25.2; Urban: 74.8 & Rural: 39.94; Urban: 28.12 & 11.82 & {[}5,90{]} \\ \hline
EIE (Physics) & 1045 & 8.33 & Sex & Female: 14.64; Male: 85.36 & Female: 10.46; Male: 7.96 & 2.5 & All \\ \hline
EIE (Physics) & 1045 & 8.33 & Territory & Rural: 24.4; Urban: 75.6 & Rural: 12.16; Urban: 7.09 & 5.07 & {[}5,20{]} \\ \hline
EIE (Chemsitry) & 1030 & 10.68 & Sex & Female: 63.4; Male: 36.6 & Female: 11.18; Male: 9.81 & 1.37 & All \\ \hline
EIE (Chemsitry) & 1030 & 10.68 & Territory & Rural: 22.23; Urban: 77.77 & Rural: 15.72; Urban: 9.24 & 6.48 & {[}5,25{]} \\ \hline
EIE (Geography) & 1132 & 5.3 & Sex & Female: 48.23; Male: 51.77 & Female: 4.21; Male: 6.31 & 2.1 & All \\ \hline
EIE (Geography) & 1132 & 5.3 & Territory & Rural: 27.47; Urban: 72.53 & Rural: 6.43; Urban: 4.87 & 1.56 & All \\ \hline
EIE (English) & 1278 & 10.64 & Sex & Female: 50.78; Male: 49.22 & Female: 9.55; Male: 11.76 & 2.21 & All \\ \hline
EIE (English) & 1278 & 10.64 & Territory & Rural: 15.34; Urban: 84.66 & Rural: 17.35; Urban: 9.43 & 7.92 & {[}5,20{]} \\ \hline
EIE (German) & 1004 & 11.35 & Sex & Female: 65.54; Male: 34.46 & Female: 10.03; Male: 13.87 & 3.84 & {[}5,90{]} \\ \hline
\textit{EIE (German)} & 1004 & 11.35 & Territory & Rural: 15.14; Urban: 84.86 & Rural: 26.97; Urban: 8.57 & 18.4 & None \\ \hline
Folktables (Employment) & 1970 & 46.45 & Sex & Female: 51.68; Male: 48.32 & Female: 44.3; Male: 48.74 & 4.44 & All \\ \hline
Folktables (Employment) & 1970 & 46.45 & Race & Asian alone: 8.63; Black or African American alone: 12.18; Other: 8.88;   White: 70.3 & Asian alone: 50.0; Black or African American alone: 42.08; Other: 41.71;   White: 47.36 & 8.29 & All \\ \hline
Folktables (Income) & 1031 & 41.51 & Sex & Female: 49.37; Male: 50.63 & Female: 35.76; Male: 47.13 & 11.37 & All \\ \hline
Folktables (Income) & 1031 & 41.51 & Race & Asian alone: 9.21; Black or African American alone: 11.25; Other: 8.05;   White: 71.48 & Asian alone: 41.05; Black or African American alone: 31.9; Other: 25.3;   White: 44.91 & 19.61 & {[}5,50{]} \\ \hline
Folktables (Medical Cover) & 1352 & 40.09 & Sex & Female: 56.14; Male: 43.86 & Female: 38.87; Male: 41.65 & 2.78 & All \\ \hline
Folktables (Medical Cover) & 1352 & 40.09 & Race & Asian alone: 10.43; Black or African American alone: 15.24; Other: 11.76;   White: 62.57 & Asian alone: 41.13; Black or African American alone: 52.43; Other: 49.06;   White: 35.22 & 17.21 & {[}5,60{]} \\ \hline
Folktables (Mobility) & 1214 & 78.17 & Sex & Female: 50.41; Male: 49.59 & Female: 77.29; Male: 79.07 & 1.78 & {[}5,70{]} \\ \hline
Folktables (Mobility) & 1214 & 78.17 & Race & Asian alone: 10.79; Black or African American alone: 13.67; Other: 10.79;   White: 64.74 & Asian alone: 75.57; Black or African American alone: 81.33; Other: 81.68;   White: 77.35 & 6.11 & {[}5,60{]} \\ \hline
Folktables (Travel Time) & 1824 & 53.78 & Sex & Female: 49.4; Male: 50.6 & Female: 51.72; Male: 55.8 & 4.08 & All \\ \hline
Folktables (Travel Time) & 1824 & 53.78 & Race & Asian alone: 9.27; Black or African American alone: 11.02; Other: 8.06;   White: 71.66 & Asian alone: 66.86; Black or African American alone: 69.15; Other: 64.63;   White: 48.51 & 20.64 & {[}5,55{]} \\ \hline
Loan Assessment & 1200 & 22.17 & Education level & High school: 16.42; Not high school or university or greater: 1.58;   University or greater: 82.0 & High school: 25.38; Not high school or university or greater: 10.53;   University or greater: 21.75 & 14.85 & None \\ \hline
Loan Assessment & 1200 & 22.17 & Sex & Female: 60.33; Male: 39.67 & Female: 20.86; Male: 24.16 & 3.3 & All \\ \hline
Diabetes Risk Assessment & 520 & 61.54 & Sex & Female: 36.92; Male: 63.08 & Female: 90.1; Male: 44.82 & 45.28 & {[}5,10{]} \\ \hline
Student Performance & 1003 & 22.03 & Sex & Female: 56.63; Male: 43.37 & Female: 21.13; Male: 23.22 & 2.09 & All \\ \hline
Student Performance & 1003 & 22.03 & Address & Rural: 27.42; Urban: 72.58 & Rural: 27.27; Urban: 20.05 & 7.22 & All \\ \hline
Student Performance & 1003 & 22.03 & Parent's education level & High school: 24.13; Not high school or university or greater: 40.88;   University or greater: 35.0 & High school: 23.55; Not high school or university or greater: 25.85;   University or greater: 16.52 & 9.33 & {[}5,55{]} \\ \hline
\end{tabular}
\begin{flushleft}
\hspace{0.35cm}Notes: $^1$\rc is expressed as the percentage of the number of samples
\end{flushleft}
\end{table}

\pagebreak

\begin{table}[H]
\small
\caption{Experimental results for FolkTables with intersectionality}
\label{tab:exp_intersect_results}
\begin{tabular}{|p{1.4cm}H|p{1.2cm}|p{1.5cm}H|p{5.5cm}|p{1.8cm}|p{1cm}|}
\hline
\textbf{Dataset (Outcome)} & \textbf{Number of Samples} & \textbf{Overall Prevelance (\%)} & \textbf{Sensitive Attribute} & \textbf{Group Distribution (\%)} & \textbf{Group Prevelance (\%)} & \textbf{Maximum Prevelance Difference (\%)} & \textbf{Optimal \rc Range} \\ \hline 
Folktables (Employment) & 1970 & 46.45 & Race and sex & Asian alone, Female: 4.52; Asian alone, Male: 4.11; Black or African American alone, Female: 6.55; Black or African American alone, Male: 5.63; Other; Female: 4.62; Other; Male: 4.26; White, Female: 35.99; White, Male: 34.31 & Asian alone, Female: 46.07; Asian alone, Male: 54.32; Black or African American alone, Female: 44.19; Black or African American alone, Male: 39.64; Other; Female: 39.56; Other; Male: 44.05; White, Female: 44.71; White, Male: 50.15 & 14.76 & All \\ \hline
Folktables (Income) & 1031 & 41.51 & Race and sex & Asian alone, Female: 4.46; Asian alone, Male: 4.75; Black or African American alone, Female: 6.3; Black or African American alone, Male: 4.95; Other; Female: 3.98; Other; Male: 4.07; White, Female: 34.63; White, Male: 36.86 & Asian alone, Female: 39.13; Asian alone, Male: 42.86; Black or African American alone, Female: 30.77; Black or African American alone, Male: 33.33; Other; Female: 21.95; Other; Male: 28.57; White, Female: 37.82; White, Male: 51.58 & 29.63 & {[}5,35{]} \\ \hline
Folktables (Public Medical Coverage) & 1352 & 40.09 & Race and sex & Asian alone, Female: 6.07; Asian alone, Male: 4.36; Black or African American alone, Female: 7.84; Black or African American alone, Male: 7.4; Other; Female: 6.51; Other; Male: 5.25; White, Female: 35.72; White, Male: 26.85 & Asian alone, Female: 40.24; Asian alone, Male: 42.37; Black or African American alone, Female: 53.77; Black or African American alone, Male: 51.0; Other; Female: 51.14; Other; Male: 46.48; White, Female: 33.13; White, Male: 38.02 & 20.64 & All \\ \hline
Folktables (Mobility) & 1214 & 78.17 & Race and sex & Asian alone, Female: 5.52; Asian alone, Male: 5.27; Black or African American alone, Female: 7.0; Black or African American alone, Male: 6.67; Other; Female: 5.44; Other; Male: 5.35; White, Female: 32.45; White, Male: 32.29 & Asian alone, Female: 74.63; Asian alone, Male: 76.56; Black or African American alone, Female: 81.18; Black or African American alone, Male: 81.48; Other; Female: 81.82; Other; Male: 81.54; White, Female: 76.14; White, Male: 78.57 & 7.19 & {[}5,60{]} \\ \hline
Folktables (Travel Time) & 1824 & 53.78 & Race and sex & Asian alone, Female: 4.5; Asian alone, Male: 4.77; Black or African American alone, Female: 6.25; Black or African American alone, Male: 4.77; Other; Female: 4.0; Other; Male: 4.06; White, Female: 34.65; White, Male: 37.01 & Asian alone, Female: 65.85; Asian alone, Male: 67.82; Black or African American alone, Female: 69.3; Black or African American alone, Male: 68.97; Other; Female: 63.01; Other; Male: 66.22; White, Female: 45.41; White, Male: 51.41 & 23.89 & {[}5,80{]} \\ \hline
\end{tabular}
\begin{flushleft}
\end{flushleft}
\end{table}

\pagebreak

\begin{table}[]
\caption{Sample size sensitivity analysis}
\label{tab:exp_results_sample_sa}
\small
\begin{tabular}
{|p{3cm}|p{1cm}|p{1.5cm}|p{1.2cm}H|p{3.7cm}|p{1.8cm}|p{1cm}|}
\hline
\textbf{Dataset (Outcome)} & \textbf{Number of Samples} & \textbf{Overall Prevelance (\%)} & \textbf{Sensitive Attribute} & \textbf{Group Distribution (\%)} & \textbf{Group Prevelance (\%)} & \textbf{Maximum Group Prevelance Difference (\%)} & \textbf{Optimal \rc Range} \\ \hline
EIE (Ukranian) & 1445 & 7.34 & Sex & Female: 51.97; Male: 48.03 & Female: 4.26; Male: 10.66 & 6.4 & None \\ \hline
EIE (Ukranian) & 1445 & 7.34 & Territory & Rural: 25.33; Urban: 74.67 & Rural: 10.38; Urban: 6.3 & 4.08 & {[}5,20{]} \\ \hline
EIE (History) & 1994 & 18 & Sex & Female: 54.36; Male: 45.64 & Female: 14.48; Male: 22.2 & 7.72 & {[}5,60{]} \\ \hline
EIE (History) & 1994 & 18 & Territory & Rural: 29.69; Urban: 70.31 & Rural: 20.27; Urban: 17.05 & 3.22 & All \\ \hline
EIE (Math) & 1222 & 31.1 & Sex & Female: 48.69; Male: 51.31 & Female: 30.92; Male: 31.26 & 0.34 & All \\ \hline
EIE (Math) & 1222 & 31.1 & Territory & Rural: 25.2; Urban: 74.8 & Rural: 39.94; Urban: 28.12 & 11.82 & {[}5,90{]} \\ \hline
EIE (Ukranian) & 5777 & 7.34 & Sex & Female: 51.96; Male: 48.04 & Female: 4.26; Male: 10.67 & 6.41 & None \\ \hline
EIE (Ukranian) & 5777 & 7.34 & Territory & Rural: 25.34; Urban: 74.66 & Rural: 10.38; Urban: 6.31 & 4.07 & {[}5,20{]} \\ \hline
EIE (History) & 5982 & 17.99 & Sex & Female: 54.35; Male: 45.65 & Female: 14.43; Male: 22.23 & 7.8 & {[}5,60{]} \\ \hline
EIE (History) & 5982 & 17.99 & Territory & Rural: 29.69; Urban: 70.31 & Rural: 20.21; Urban: 17.05 & 3.16 & All \\ \hline
EIE (Math) & 7327 & 31.05 & Sex & Female: 48.7; Male: 51.3 & Female: 30.91; Male: 31.18 & 0.27 & All \\ \hline
EIE (Math) & 7327 & 31.05 & Territory & Rural: 25.14; Urban: 74.86 & Rural: 39.74; Urban: 28.13 & 11.61 & {[}5,90{]} \\ \hline
EIE (Ukranian) & 51991 & 7.34 & Sex & Female: 51.97; Male: 48.03 & Female: 4.26; Male: 10.68 & 6.42 & None \\ \hline
EIE (Ukranian) & 51991 & 7.34 & Territory & Rural: 25.34; Urban: 74.66 & Rural: 10.37; Urban: 6.31 & 4.06 & {[}5,20{]} \\ \hline
EIE (History) & 51837 & 17.98 & Sex & Female: 54.34; Male: 45.66 & Female: 14.43; Male: 22.2 & 7.77 & {[}5,60{]} \\ \hline
EIE (History) & 51837 & 17.98 & Territory & Rural: 29.69; Urban: 70.31 & Rural: 20.21; Urban: 17.03 & 3.18 & All \\ \hline
EIE (Math) & 51283 & 31.05 & Sex & Female: 48.68; Male: 51.32 & Female: 30.92; Male: 31.18 & 0.26 & All \\ \hline
EIE (Math) & 51283 & 31.05 & Territory & Rural: 25.14; Urban: 74.86 & Rural: 39.75; Urban: 28.13 & 11.62 & {[}5,90{]} \\ \hline
\end{tabular}
\end{table}

\clearpage

\begin{table}[H]
\caption{\ppv sensitivity analysis}
\label{tab:exp_results_ppv_sa}
\small
\begin{tabular}
{|p{3cm}|p{1.5cm}|p{1.2cm}|p{1.8cm}|p{1.5cm}|p{1.5cm}|}
\hline
\textbf{Dataset (Outcome)} & \textbf{Overall Prevelance (\%)} & \textbf{Sensitive Attribute} & \textbf{Maximum Group Prevelance Difference (\%)} & \textbf{Optimal \rc Range \linebreak (\ppv = 0.7)} & \textbf{Optimal \rc Range \linebreak (\ppv = 0.85)} \\ \hline
EIE (Ukranian) & 7.34 & Sex & 6.4 & None & None \\ \hline
EIE (Ukranian) & 7.34 & Territory & 4.08 & {[}5,20{]} & {[}5,20{]} \\ \hline
EIE (History) & 18 & Sex & 7.72 & {[}5,60{]} & {[}5,60{]} \\ \hline
EIE (History) & 18 & Territory & 3.22 & All & All \\ \hline
EIE (Math) & 31.1 & Sex & 0.34 & All & All \\ \hline
EIE (Math) & 31.1 & Territory & 11.82 & {[}5,90{]} & {[}5,90{]} \\ \hline
EIE (Physics) & 8.33 & Sex & 2.5 & All & {[}5,90{]} \\ \hline
EIE (Physics) & 8.33 & Territory & 5.07 & {[}5,20{]} & {[}5,20{]} \\ \hline
EIE (Chemsitry) & 10.68 & Sex & 1.37 & All & All \\ \hline
EIE (Chemsitry) & 10.68 & Territory & 6.48 & {[}5,25{]} & {[}5,25{]} \\ \hline
EIE (Geography) & 5.3 & Sex & 2.1 & All & {[}5,90{]} \\ \hline
EIE (Geography) & 5.3 & Territory & 1.56 & All & All \\ \hline
EIE (English) & 10.64 & Sex & 2.21 & All & All \\ \hline
EIE (English) & 10.64 & Territory & 7.92 & {[}5,20{]} & {[}5,20{]} \\ \hline
EIE (German) & 11.35 & Sex & 3.84 & {[}5,90{]} & {[}5,90{]} \\ \hline
EIE (German) & 11.35 & Territory & 18.4 & None & {[}5,10{]} \\ \hline
Folktables (Employment) & 46.45 & Sex & 4.44 & All & All \\ \hline
Folktables (Employment) & 46.45 & Race & 8.29 & All & All \\ \hline
Folktables (Income) & 41.51 & Sex & 11.37 & All & All \\ \hline
Folktables (Income) & 41.51 & Race & 19.61 & {[}5,50{]} & {[}5,55{]} \\ \hline
Folktables (Public Medical Coverage) & 40.09 & Sex & 2.78 & All & All \\ \hline
Folktables (Public Medical Coverage) & 40.09 & Race & 17.21 & {[}5,60{]} & {[}5,60{]} \\ \hline
Folktables (Mobility) & 78.17 & Sex & 1.78 & {[}5,70{]} & All \\ \hline
Folktables (Mobility) & 78.17 & Race & 6.11 & {[}5,60{]} & All \\ \hline
Folktables (Travel Time) & 53.78 & Sex & 4.08 & All & All \\ \hline
Folktables (Travel Time) & 53.78 & Race & 20.64 & {[}5,55{]} & {[}5,70{]} \\ \hline
Loan Assessment & 22.17 & Education level & 14.85 & None & None \\ \hline
Loan Assessment & 22.17 & Sex & 3.3 & All & All \\ \hline
Diabetes Risk Assessment & 61.54 & Sex & 45.28 & {[}5,10{]} & {[}5,10{]} \\ \hline
Student Performance & 22.03 & Sex & 2.09 & All & All \\ \hline
Student Performance & 22.03 & Address & 7.22 & All & All \\ \hline
Student Performance & 22.03 & Parent's education level & 9.33 & {[}5,55{]} & {[}5,55{]} \\ \hline
\end{tabular}
\end{table}

\section{List of all metrics and their equations}
\label{appendix:metrics}
\paragraph{Metrics} A traditional confusion matrix is a standard tool for understanding binary classification tasks, and details potential model outcomes, giving names to the relation between what the model predicts and what the ground truth labels actually are - see Table~\ref{tab:conf_matrix}. 
\begin{table}[]
    \centering
    \begin{tabular}{cc|cc}
\multicolumn{2}{c}{}
            &   \multicolumn{2}{c}{Actual} \\
    &       &   Positive (1) &   Negative (0)             \\ 
    \cline{2-4}
\multirow{2}{*}{\rotatebox[origin=c]{90}{Predicted}}
    & Positive (1)   & TP   & FP                 \\
    & Negative (0)   & FN    & TN                \\ 
    \cline{2-4}
    \end{tabular}
    \caption{Standard confusion matrix.}
    \label{tab:conf_matrix}
\end{table}
From the confusion matrix comes a set of standard metrics that capture relationships in the outcomes of a binary classification model. Here we detail potentially relevant metrics to this paper.
\begin{itemize}
    \item $TPR = \frac{TP}{TP + FN}$
    \item $TNR = \frac{TN}{TN + FP}$
    \item $FPR = \frac{FP}{FP + TN}$
    \item $FNR = \frac{FN}{FN+TP}$
    \item $PPV = \frac{TP}{TP+FP}$
    \item $\acc = \frac{TP+TN}{TP+FP+TN+FN}$
\end{itemize}

\paragraph{Binary Classification} We define binary classification as follows. Consider dataset $X$ consisting of observations for individuals $x_1,x_2,...,x_n$. For individual $x\in X$ with a target value of interest (or ``label'') $y \in Y$, where $Y \in \{0,1\}^n$, we seek to \textit{classify} $x$ correctly (i.e. assign label $y$ to $x$). More formally, we attempt to learn a function $\hat{Y}$ such that $\forall (x_i, y_i) \in X \cup Y$, $\hat{Y}(x_i) \rightarrow \hat{y_i}$ and $\hat{y_i} == y_i$.

Often, we utilize a statistical technique to find a function $\hat{Y}$ that produces a real valued score $\Tilde{y} \in (0,1)$, and to binarize the outputs we apply a standard thresholding function $\tau(\Tilde{y}) = \hat{y}$ to produce a binary label $\hat{y} \in \{0,1\}$. For example, if $\Tilde{y}$ is interpreted as a \textit{probability} of 1 being the correct label (a ``positive'' assignment), then $\tau$ thresholding at a greater than $\frac{1}{2}$ probability of positive class assignment is a way to convert from score or probability to concrete class.
$$\tau(\Tilde{y}) =\left\{\begin{array}{lr}
        1, & \text{for } \frac{1}{2} \leq \Tilde{y} \leq 1\\
        0, & \text{otherwise}
        \end{array}\right\}$$

However, in situations with a resource constraint $k$ that governs how many positive labels we are allowed to assign (say, in a college admissions scenario), we may be forced to adjust $\tau(\Tilde{y})$ to accept a function of $k$, i.e., $f(k) = t$ where $t \in [0,1]$ such that $\tau(\Tilde{y}, f(k))$ produces exactly $k$ positive classifications. For example: $$\tau(\Tilde{y}, f(k)) =\left\{\begin{array}{lr}
        1, & \text{for } f(k) \leq \Tilde{y} \leq 1\\
        0, & \text{otherwise}
        \end{array}\right\} s.t. \sum_{i=1}^n \mathbf{1}(\hat{y}_i = 1) = k$$

\paragraph{Fairness Considerations: Binary Classification with Sensitive Features} Often, when considering the algorithmic fairness of a binary classifier, we consider a \textit{sensitive} or \textit{protected} attribute in the data that denotes \textit{group membership}. For example, many datasets collected in social settings have information about the \textit{race} or \textit{gender} of individuals in the population. Both of these attributes are and should be ``protected,'' morally and lawfully. Thus, when we evaluate our binary classifier (say, along \fpr or \fnr), we can evaluate each metric for the entire population, and we can also evaluate each metric \textit{conditioned on group membership}. In the simplest case (which is our focus for much of this paper), our \textit{sensitive} attribute is binary, and thus we consider $\fpr_1$ and $\fpr_2$,  $\fnr_1$ and $\fnr_2$, etc. (metrics evaluated on the disjoint sets of outcomes based on group conditioning). 

\section{Analytical approach to characterizing the fairness region using \fpr, \fnr, \acc}
\label{appendix:imp_acc}

\begin{corollary}[Impossibility Result Variation \cite{DBLP:journals/bigdata/Chouldechova17} \cite{DBLP:conf/innovations/KleinbergMR17}]
Consider a binary classification setting. The relationship between \acc, \fnr, \fpr and $p$ can be characterized by:

\begin{align}
    ACC = (1 - FNR)p + (1 - FPR)(1 - p)
\end{align}
\label{theorem:acc_impossiblity}
\end{corollary}

\begin{proof}
Consider the following statements over accuracy, and note that they apply overall as well as for for some Group $i$. Thus, each of these quantities could be subscripted with $i$ i.e. $ACC = ACC_i$, etc.
\begin{align}
    ACC &= \frac{TP + TN}{TP+TN+FP+FN} \\
    FNR &= 1-\frac{TP}{TP+ FN}\\
    FPR &= 1-\frac{TN}{TN+ FP}\\
    p &= \frac{TP + FN}{TP+TN+FP+FN} \\
    1-p &= \frac{TN + FP}{TP+TN+FP+FN}
\end{align}
Split the numerator for $ACC$, and multiply by a clever 1:
\begin{align}
    ACC &= \frac{TP}{TP+TN+FP+FN} + \frac{TN}{TP+TN+FP+FN} \\
    &= \frac{TP+FN}{TP+FN} \times \frac{TP}{TP+TN+FP+FN} + \frac{TN+FP}{TN+FP} \times \frac{TN}{TP+TN+FP+FN} \\
    &= \frac{TP}{TP+FN} \times \frac{TP+FN}{TP+TN+FP+FN} + \frac{TN}{TN+FP} \times \frac{TN+FP}{TP+TN+FP+FN} \\
    &= (1-FNR)p + (1-FPR)(1-p)
\end{align}
\end{proof}


\begin{lemma}[Expressing Fairness Area Variation]
Consider Corrolary~\ref{theorem:acc_impossiblity}. Assume that $p_2 = p_1 + \epsilon_p$, $\acc_2 = \acc_1 + \epsilon_\acc$, $\fpr_2 = \fpr_1 + \epsilon_\fpr$, and $\fnr_2 = \fnr_1 + \epsilon_\fnr$, where each $\epsilon_\fpr, \epsilon_\acc, \epsilon_\fnr, \epsilon_p \in (-1,1)$ term captures the difference between two groups for $FPR$, $ACC$, $FNR$, $p$ respectively. Then the following equality holds:
\begin{align}
\label{eq:variation_area_exp}
\fnr = \frac{-\epsilon_\fpr + \epsilon_\acc + \epsilon_\fpr~~p - \epsilon_\fnr~~p + \alpha \epsilon_p + \epsilon_\fpr~~\epsilon_p - \epsilon_\fnr~~\epsilon_p}{\epsilon_P}
\end{align}

\label{lemma:constraints}
\end{lemma}

\begin{proof} Consider \textit{Corollary}~\ref{theorem:acc_impossiblity} in the setting where there are two groups. Suppose $ACC_1 = ACC_2$. Then:
\begin{align}
    (1-FNR_1)p_1 + (1-FPR_1)(1-p_1) = (1-FNR_2)p_2 + (1-FPR_2)(1-p_2) 
\end{align}

Make substitutions respective to the assumptions made in Lemma~\ref{lemma:constraints} to find:
\begin{align}
    (1 - \fnr)p + (1-\fpr)(1-p) = (1 - (\fnr+\epsilon_\fnr)(p+\epsilon_p)) + (1-(\fpr+\epsilon_\fpr))(1-(p+\epsilon_p)) + \epsilon_\acc
\end{align}

Solving for $\fnr$ yields (\ref{eq:variation_area_exp}).


\end{proof}


\begin{lemma}[Closed-Form for Fairness Area Variation] \label{lemma:area_closed_form} 
Assume $\epsilon_p < 1 - p$. Allow $\pm \gamma$ to be the symmetric acceptable error (our ``fair'' relaxation) between groups for metrics $FPR$, $FNR$ and $ACC$. Consider the size of the space of possible $\epsilon_\fpr, \epsilon_\fnr, \epsilon_\acc$ assignments, given $\epsilon_p$ and $p$, that satisfy the constraints from Lemma~\ref{lemma:constraints}. We will denote the size of that space as $|A_f|$ (as a shorthand, we will call that the ``fairness region'', but the reality is more nuanced). For a set of fairness constraints $\epsilon_\fpr, \epsilon_\fnr, \epsilon_\acc \in (-\gamma,\gamma)$, where $|\gamma| \leq 1$ and $\gamma \neq 0$, we have that $|A_f|$ is simply:
\begin{align}
    |A_f| = \frac{4 \gamma}{\epsilon_p} - \frac{4 \gamma^2}{{\epsilon_p}^2}
\end{align} 

\end{lemma}


Before proving Lemma~\ref{lemma:area_closed_form}, let's 
 briefly motivate the three primary assumptions: first, we should expect $\epsilon_p << p$, as our relaxation constant should not really be on the same order of magnitude as our per-group prevalence (think $p \approx 0.5$ and $\epsilon_p \approx 0.05$). Thus, the assumption that $\epsilon_p < 1 - p$ is very reasonable. 

Second, pre-specifying an acceptable $\gamma$ relaxation term may seem odd, but it is very common among practitioners, who prefer small groups variations to large ones. Thus, think of $\gamma$ as a small value, something like $\gamma \leq 0.05$.

Third, assuming that $|\gamma| > 0$ is necessary, as when $|\gamma| = 0$ we recover Corollary~\ref{theorem:acc_impossiblity}. We also ignore the case where $\epsilon_p = 0$ because the implications of the impossibility theorem do not apply in the case of equal base rates.


\begin{proof}
We begin with the result from Lemma~\ref{lemma:constraints}. Rearranging terms, we find the following expression for $\fpr$:
\begin{align}
    \label{eq:c}
    \left(\frac{\epsilon_\acc + \epsilon_\fpr - p(\epsilon_\fpr -\epsilon_\fnr)}{\epsilon_p} + \epsilon_\fnr - \epsilon_\fpr \right) + \fnr = \fpr
\end{align}

Set $c= \fpr - \fnr = \left(\frac{\epsilon_\acc + \epsilon_\fpr - p(\epsilon_\fpr -\epsilon_\fnr)}{\epsilon_p} + \epsilon_\fnr - \epsilon_\fpr \right)$, which is fixed for prevalence $p$, prevalence difference $\epsilon_p$, and a set relaxation factors $\epsilon_\acc, \epsilon_\fnr, \epsilon_\fpr$.

It's clear that the relationship between $\fpr$ and $\fnr$ is linear, and controlled by $c$, which can take on many possible values as we vary the relaxation parameters. We notate the set of values that $c$ can take on as $C=\{c_1,c_2...c_m\}$. $C$ is an infinite set.

However, $C$ contains maximum and minimum values. From the linear relationship between $\fpr$ and $\fnr$, we have $c_{max} = max(C)$ and $c_{min} = min(C)$. $c_{max}$ and $c_{min}$ will help us define the boundaries of the solution space $A_f$. It follows that $|A_f|$, the size of the solution space, when $|c| < 1$ and $\fpr,\fnr < 1$, is: 

\begin{figure}
    \centering
    \includegraphics[width=5cm]{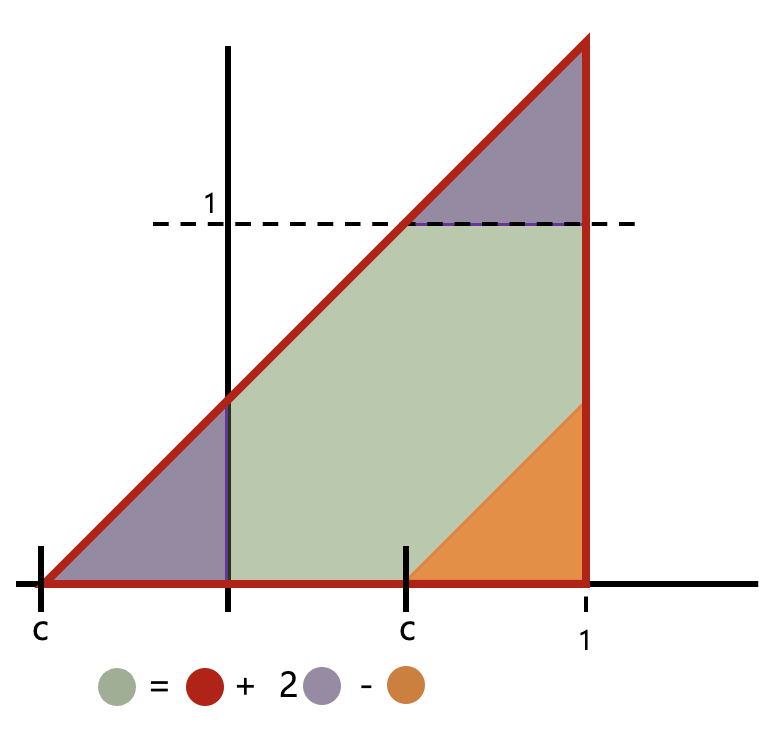}
    \caption{A sketch of the integral construction.}
    \label{fig:sketch_window_acc}
\end{figure}

\begin{align}
    |A_f| = \int_{-c_{max}}^1 (c_{max} + \fnr) \text{ }d\fnr - \left( 2 \int_{-c_{max}}^0 (c_{max} + \fnr) \text{ }d\fnr \right) - \left( \int_{-c_{min}}^1 (c_{min} + \fnr) \text{ }d\fnr \right) 
\end{align}
For a sketch of the integral construction, see Figure~\ref{fig:sketch_window_acc}.

Using our construction over $\epsilon_\acc, \epsilon_\fpr, \epsilon_\fnr \in (-\gamma,\gamma)$, we can deduce $c_{max}$ and $c_{min}$.
\begin{align}
    c &= \frac{\epsilon_\acc + \epsilon_\fpr - p(\epsilon_\fpr -\epsilon_\fnr)}{\epsilon_p} + \epsilon_\fnr - \epsilon_\fpr \\
    &= \frac{\epsilon_\acc}{\epsilon_p} + \frac{\epsilon_\fpr}{\epsilon_p} - \frac{p\epsilon_\fpr}{\epsilon_p} + \frac{p\epsilon_\fnr}{\epsilon_p} + \epsilon_\fnr - \epsilon_\fpr \\
    &= \frac{\epsilon_\acc}{\epsilon_p} + \frac{\epsilon_\fpr}{\epsilon_p} - \frac{p\epsilon_\fpr}{\epsilon_p} - \epsilon_\fpr + \frac{p\epsilon_\fnr}{\epsilon_p} + \epsilon_\fnr \\
    &= \frac{\epsilon_\acc}{\epsilon_p} + \epsilon_\fpr \left(\frac{1-p}{\epsilon_p} - 1 \right) + \frac{p\epsilon_\fnr}{\epsilon_p} + \epsilon_\fnr \\
    \text{Note our assumption that}&\text{ $\epsilon_p < 1 - p$ gives us that }\frac{1-p}{\epsilon_p}>1 \text{ which yields:}\\
    &\leq \frac{\gamma}{\epsilon_p} + \gamma \left(\frac{1-p}{\epsilon_p} - 1 \right) + \frac{p\gamma}{\epsilon_p} + \gamma \\
    &=\frac{\gamma + \gamma - \gamma p + \gamma p }{\epsilon_p} = \frac{2\gamma}{\epsilon_p} = c_{max}
\end{align}
A symmetric argument gives $c_{min} = -\frac{2\gamma}{\epsilon_p}$. To find $|A_f|$, we set $c = \frac{2\gamma}{\epsilon_p}$, and replace into the above integration:
\begin{align}
    |A_f| = \int_{-c}^1 c + \fnr \text{ }d\fnr - \left( 2 \int_{-c}^0 c + \fnr \text{ }d\fnr \right) - \left( \int_{c}^1 \fnr - c \text{ }d\fnr \right) = 2c - c^2 = \frac{4 \gamma}{\epsilon_p} - \frac{4 \gamma^2}{{\epsilon_p}^2} \leq 1
\end{align}
This yields the result.
\end{proof}

\section{Analytical approach to characterizing the fairness region area using \fpr, \fnr, \ppv}
\label{appendix:region_ppv}

\begin{lemma}[Expressing Fairness Area]  Consider groups $\grp_1$ with prevalence $\prev_1$ and $\grp_2$ with prevalence $\prev_2$.  Without loss of generality, let us assume that $\prev_2 = \prev_1 + \epsilon_\prev$.  Next, let us denote a predictor's performance as $\fpr_1$, $\fpr_1$ and $\ppv_1$ for $\grp_1$, and $\fpr_1$, $\fpr_1$ and $\ppv_1$ for $\grp_2$.  Let $\epsilon_\fpr$, $\epsilon_\fnr$ and $\epsilon_v$ denote acceptable differences in \fpr, \fnr and \ppv between groups, respectively.  That is,  $|\fpr_1 - \fpr_2| \leq \epsilon_\fpr$, $|\fnr_1 - \fnr_2| \leq \epsilon_\fnr$ and $|\ppv_1 - \ppv_2| \leq \epsilon_v$.
Then, the following equality holds (for the sake of space, let $\fpr = \alpha, \fnr = \beta, and \ppv = v$): 
{
\small 
\begin{align}
    \label{eq:beta}
    \beta = \frac{\epsilon_p (v^2 (\epsilon_\fpr (p - 1) - 1) + v \epsilon_v (\epsilon_\fpr (p - 1) - 1) + p \epsilon_v + v) + (p - 1) (\epsilon_\fpr (p - 1) v (v + \epsilon_v) + p \epsilon_v) - \epsilon_\fnr (p - 1) v (p + \epsilon_p) (v + \epsilon_v - 1)}{(\epsilon_p (p \epsilon_v - v^2 - v \epsilon_v + v) + (p - 1) p \epsilon_v)}
\end{align}
}
\label{lemma:relaxed_impossibility_result}
\end{lemma}
\begin{proof}

Consider (\ref{impossibility_equation}) in the setting where there are two groups. Let the \fpr of the two groups be equal, subject to the relaxation $\fpr_2 = \fpr_1 + \epsilon_\fpr$. Make substitutions respective to the assumptions made in Lemma \ref{lemma:relaxed_impossibility_result} to find:

\begin{align}
    \label{fully_relaxed}
    \frac{p}{1-p} \frac{1-v}{v} (1 - \beta) = \frac{p+\epsilon_p}{1-(p+\epsilon_p)} \frac{1-(v+\epsilon_v)}{v+\epsilon_v} (1 - (\beta + \epsilon_\fnr)) + \epsilon_\fpr
\end{align}

Solving for $\beta$ yields (\ref{eq:beta}).
\end{proof}

\section{Analysis of the Dot Planimeter}
\label{appendix:dot}
One of the strategies we undertook in estimating the ``fairness region'' of~\ref{lemma:relaxed_impossibility_result} borrowed from previous work on ``dot-planimetry,'' a two-dimensional area estimation technique that has been consistently reinvented over the past century, and has been studied extensively in relation to cartographic area estimation and pure math (see Gauss's circle problem) \cite{frolov1969accuracy,bocarov1957matematiko, lowry2017some} (Perhaps sadly, GPS has contributed to the under-exploration of this subject in recent years). Here, we offer a brief derivation of our dot-planimetry strategy, and how we derived approximate upper bounds on our over-estimation error for the area of the fairness region when considering \fnr, \fpr and \ppv.

Dot-planimetry provides a simple way of estimating the area of complex enclosed shapes in a two dimensional space (which can be difficult to integrate directly). Intuitively, we create a regular grid over the space composed of points. We refer to each point as a ``detector,'' who is responsible for a pre-specified radius $r_\epsilon$ in the 2-dimensional space of interest. We say that a ``detector'' $i$ is ``satisfied'' if the edge of the shape of interest is anywhere within $r_\epsilon$ distance away from $i$. To compute the final area of the shape, we simply sum up the total number of satisfied detectors multiplied by each of their individual areas (i.e. the area of a bunch of circles defined by $r_\epsilon$).
\begin{figure}[h!]
    \centering
    \includegraphics[width=6cm]{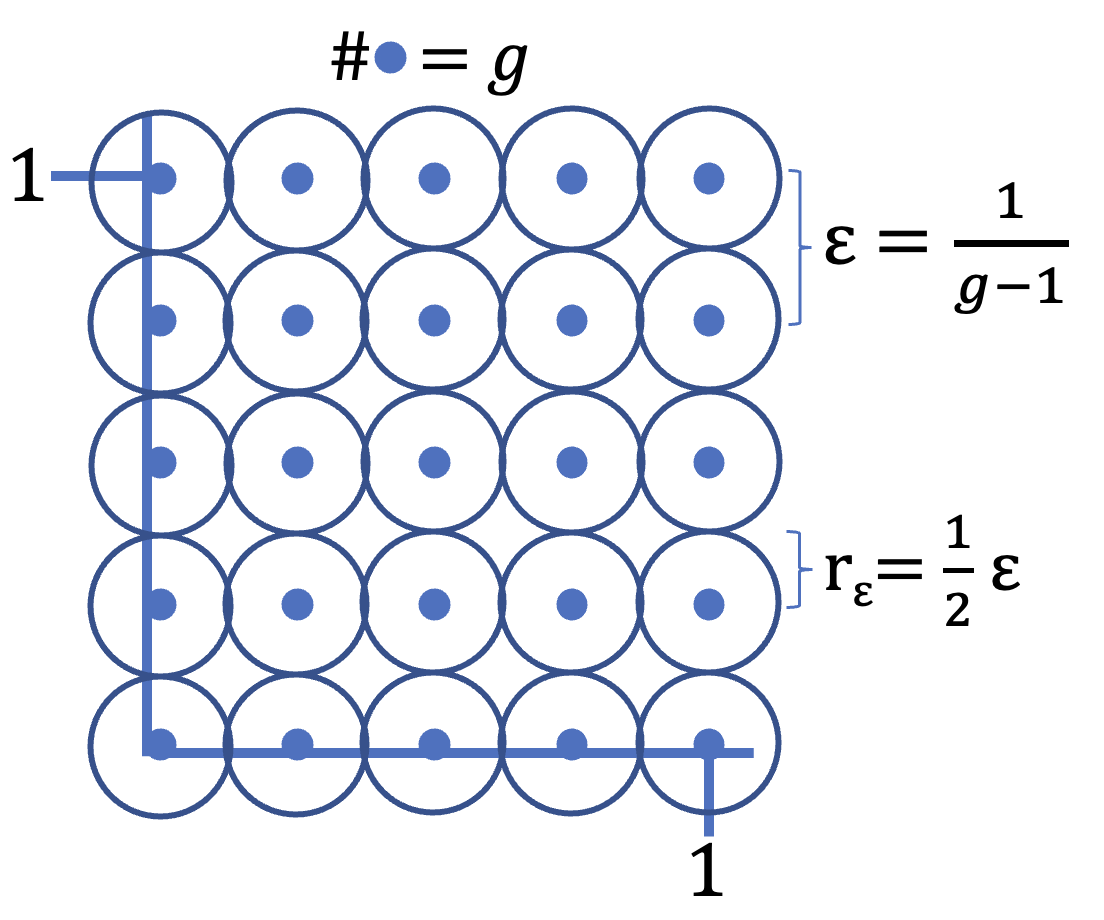}
    \caption{A full planimeter (\textit{g}=25).}
    \label{fig:full_planimeter}
\end{figure}
In our case, we are interested in a very particular space: a $1\times1$ square with the bottom left corner at the origin (this is how we can visualize two metrics varying based on our relaxations). Our dot planimeter in this case has $g^2$ total detectors. They are distributed so that they are $\epsilon=\frac{1}{g-1}$ apart, and so that the bottom row and leftmost column each touch the $x$ and $y$ axis respectively, while the rightmost and top rows each have $x$ and $y$ values of 1 respectively. Thus, each detector has a radius of $r_\epsilon=\frac{1}{2}\epsilon$. Refer to Figure~\ref{fig:full_planimeter} for a sketch of this setup, and as we walk through the problem.

For our analysis, we will call two detectors $i$ and $j$ ``neighboring'' if $i_y=j_y$ and $|i_x-j_x| = \epsilon$ i.e. the values of their $y$-axis are equal and they are next to each other. Note that, for any of the following arguments, symmetric variations apply were we to switch the definition of  ``neighboring'' to $i_x=j_x$. Neighboring detectors are shown in Figure~\ref{fig:duo_planimeter}.

We will also define our ``detector function'' to be:
$$
f(i_{xy}, h(x))=\begin{cases}
			1, & \text{$|project(i_{xy}) \rightarrow h(x)| \leq r_\epsilon$}\\
            0, & \text{otherwise}
		 \end{cases}
$$
Intuitively, our detector function $f(i_{xy}, h(x))$ takes a detector $i_{xy}$ and a boundary function of interest $h(x)$ and returns 1 (or true) if, for any $x$, $h(x)$ passes within $r_\epsilon$ of $i_{xy}$.

Our overall approach to upper bounding the over-estimation error when using a dot-planimeter is intuitive: we will reason about which function through our $[0,1]\times[0,1]$ metric region would lead to the highest number of detectors satisfied. We will then assume that this is the boundary function for our area (i.e. this function partitions our space, and is dense on one side). Our estimation error is then the proportion of total detectors satisfied by the boundary function, assuming that they all make minimal contact with the detector radius. \textbf{This is a very coarse approach, and we assume the actual overestimation error is much lower.} However, due to the computational nature of the problem, this will provide us confidence in selecting a value for for $g$ that certainly ensures an upper-bounded amount of error (like 5$\%$.

\begin{figure}[h!]
    \centering
    \includegraphics[width=2cm]{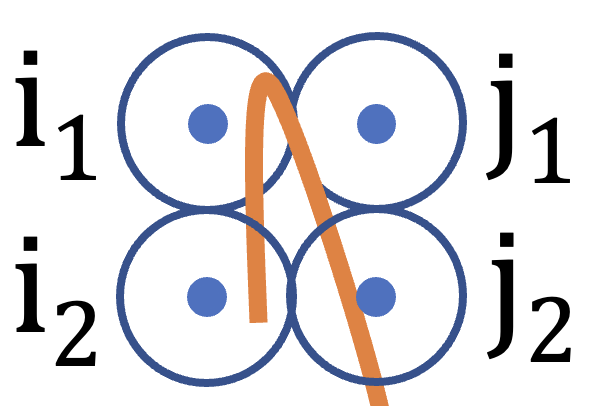}
    \caption{A set of two sets of side by side detectors, $i_1, j_1$ and $i_2, j_2$. Note that a critical point is necessary for $h(x)$ (in orange) to satisfy all 4.}
    \label{fig:duo_planimeter}
\end{figure}

\begin{proposition}[Critical points all satisfying of neighboring detectors] Consider two sets of side by side (or ``neighboring'') detectors, $i_1, j_1$ and $i_2, j_2$ (one can refer to Figure~\ref{fig:duo_planimeter}). By the definition of a function, $f(i_1, h(x)) + f(i_2, h(x)) + f(j_1, h(x)) + f(j_2, h(x)) \leq 3$ \textbf{unless} $g(x)$ has a critical point in the window $i_x \leq x \leq j_x$.
\label{prop:crit_points}
\end{proposition}

\begin{proposition}[Extending the argument to columns] Consider two columns of neighboring detectors, which can represented by sets $\{i_1, i_2,...,i_m\} \in I$ and $\{j_1, j_2,...,j_m\} \in J$. Note that detectors in $I$ share the same $x$ value, as for $J$. 

What is the maximum number of detectors satisfied by $h(x)$ between sets $x$? If we assume $|I|, |J| = g \geq 3$, then by Proposition~\ref{prop:crit_points}, if $h(x)$ has 0 critical points lie between $I_x$ and $J_x$, then the answer is at most $g+1$ (this can be seen through a geometric argument). However, we allow for any critical points, than the answer is $2g$ (or, the size of the entire union between the sets).
\label{prop:columns}
\end{proposition}

Propositions~\ref{prop:crit_points} and \ref{prop:columns} show that the determining factor in satisfying the highest number of detectors in a space is the number of critical points allowed for the function $h(x)$. The two column argument from propositions~\ref{prop:columns} can be extended to cover all detectors in the space, and provides a coarse upper bound on the over-estimation error.

\begin{proposition}[Upper bound on error of $(\epsilon)$-dot-planimeter under assumptions of fairness region] Assume that $h(x)$ is the boundary function for our area $\in [0,1]\times[0,1]$ and has at most $b$ critical points. Then, a coarse upper bound on the max detectors satisfied by $h(x)$ is $g*c$. This yields the simple \textit{percent} error calculator: $\frac{c}{g}$ (as there are $g^2$ total detectors in our dot-planimeter).
\end{proposition}
Thus, for a $5\%$ upper bound on our error, assuming no more than 6 critical points for boundary function $h(x)$ for $x\in(0,1)$, we have $\frac{6}{0.05} = g = 120$. Thus, with a granularity of $120^2 = 14400$ detector in our dot-planimeter, we can be confident that our over-estimation error is no more than $5\%$ so long as the assumptions we made about $h(x)$ hold (experimentally, this was the case). 

\begin{figure}[h!]
    \centering
    \includegraphics[width=2cm]{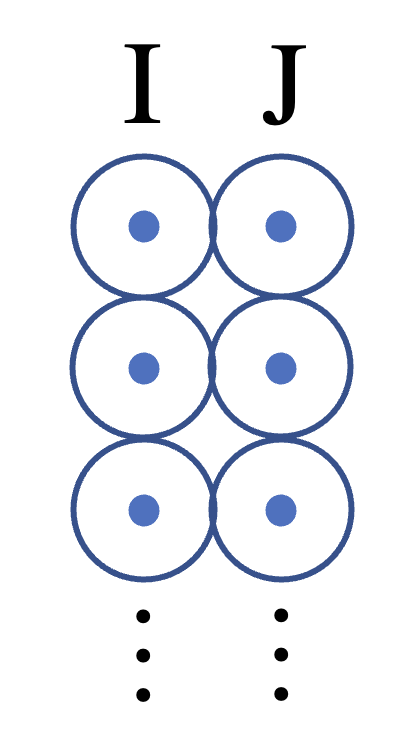}
    \caption{A set of two detector columns, $I$ and $J$.}
    \label{fig:columns_planimeter}
\end{figure}

\section{Additional Proofs}
\label{appendix:proofs}

\begin{proposition}[Intersectional Prevelance Differences]
Given a dataset that is subdivided into two groups, let $0 < p_1 < n_1$ and $0 < p_2 < n_2$, where $p_i$ is the number of positive class members of group $i$, and $n_i$ is the total number of members in group $i$. Suppose $\frac{p_1}{n_1} \leq \frac{p_2}{n_2}$. 

The following holds:
\begin{align}
    \frac{p_1}{n_1} \leq \frac{p_1 + p_2}{n_1 + n_2} \leq \frac{p_2}{n_2}
\end{align}
\end{proposition}

\begin{proof}

We know that $0 < p_1, n_1, p_2, n_2$. Note that the assumption  $\frac{p_1}{n_1} \leq \frac{p_2}{n_2}$ implies $n_2 \leq \frac{n_1 p_2}{p_1}$.

Consider the left hand side of the equality proposed in the theorem statement:
\begin{align}
    \frac{p_1}{n_1} &\leq \frac{p_1 + p_2}{n_1 + n_2} \\
    p_1 (n_1 + n_2) &\leq n_1 (p_1 + p_2) \\
    p_1 n_2 &\leq n_1 p_2 \\
    n_2 &\leq \frac{n_1 p_2}{p_1}
\end{align}
We have now derived an inequality specific to $n_2$. We can verify that exceeding this value invalidates the claim by plugging in $n_2 = (\frac{n_1 p_2}{p_1} + a)$, where $a > 0$ to find a contradiction:
\begin{align}
    p_1 (n_1 + \frac{n_1 p_2}{p_1} + a) &\leq n_1 (p_1 + p_2) \\
    p_1 n_1 + n_1 p_2 + a p_1 &\leq p_1 n_1 + n_1 p_2 \\
    a p_1 \leq 0
\end{align}
which contradicts $a > 0$. An analogous argument exists for the right hand side of the equation. Thus, so long as $n_2 \leq \frac{n_1 p_2}{p_1}$, the main inequality holds.

\end{proof}

\begin{proposition}[Reducing k increases PPV]
\label{theorem:role_of_k_app}
Given a \emph{well-calibrated classifier} being used under a resource constraint \rc, reducing the size of \rc will monotonically increase the PPV of the classifier.
\end{proposition}
\begin{proof}
Consider the top \rc outputs of a well-calibrated classifier. Since the classifier is calibrated, the list \rc is ordered in the following way: those elements closest to the first position have a higher probability of a positive outcome, and those elements closest to position \rc have a lower probability of a positive outcome. In other words, $p_1 \geq p_2 \geq, \ldots, \geq p_k$, where $p_i$ is the probability that element $i$ is a member of the positive class. The \ppv of the classifier, in expectation, can then be expressed in the following way:
\begin{align}
    \label{eq:prec_expectation}
    \frac{1}{k} \sum\limits_{i=1}^k p_i = \frac{p_1 + p_2 + ... + p_k}{k}.
\end{align}
Notice that Equation~\ref{eq:prec_expectation} is simply the average of $p_i$ over the \rc elements. Consider another set of size $k' < k$, that is constructed by removing the elements $p_{k'+1}, ... p_k$. Since by construction all the elements in the new set are \emph{greater than or equal to} to the elements of the previous set, the average of $p_i$ in the new set will be greater than or equal --- or in other words, the \ppv will be greater than or equal.
\end{proof}

\end{document}